\documentclass[twoside]{article}
\usepackage[accepted]{aistats2018}
\usepackage{natbib}
\usepackage{amsthm,amsmath,amssymb}
\usepackage{mathrsfs}

\newtheorem{lemma}{Lemma}
\newtheorem{definition}{Definition}

\DeclareMathOperator*{\argmax}{arg\,max}

\usepackage{enumitem}
\setlist[itemize]{noitemsep}
\setlist[enumerate]{noitemsep}
\usepackage[hyperfootnotes=false]{hyperref}
\usepackage{xcolor}

\newcommand{\returndists}{\mathscr{P}(\mathbb{R})^{\mathcal{X} \times \mathcal{A}}}
\newcommand{\calcd}{\mathrm{d}}

\usepackage{algorithm}
\usepackage{algorithmic}
\usepackage{bbm}

\usepackage{verbatim}
\usepackage{graphicx}

\usepackage{thmtools}
\usepackage{thm-restate}
\begin{document}

%

%
\runningauthor{Mark Rowland, Marc G. Bellemare, Will Dabney, R\'emi Munos, Yee Whye Teh}

\twocolumn[
\aistatstitle{An Analysis of Categorical Distributional Reinforcement Learning}
\setcounter{footnote}{0}
\aistatsauthor{Mark Rowland\textsuperscript{*1} \ \ \ Marc G. Bellemare$^\dagger$ \ \ \ Will Dabney$^\ddagger$ \ \ \ R\'emi Munos$^\ddagger$ \ \ \ Yee Whye Teh$^\ddagger$}
\aistatsaddress{\textsuperscript{*}University of Cambridge, $^\dagger$Google Brain, $^\ddagger$DeepMind}
 ]

\begin{abstract}
  Distributional approaches to value-based reinforcement learning model the entire distribution of returns, rather than just their expected values, and have recently been shown to yield state-of-the-art empirical performance. This was demonstrated by the recently proposed \texttt{C51} algorithm, based on categorical distributional reinforcement learning (CDRL) \citep{DistPerspective}. However, the theoretical properties of CDRL algorithms are not yet well understood. In this paper, we introduce a framework to analyse CDRL algorithms, establish the importance of the projected distributional Bellman operator in distributional RL, draw fundamental connections between CDRL and the Cram\'er distance, and give a proof of convergence for sample-based categorical distributional reinforcement learning algorithms.
\end{abstract}

\section{INTRODUCTION}

Reinforcement learning (RL) formalises the problems of evaluation and optimisation of an agent's behaviour while interacting with an environment, based upon feedback given through a reward signal \citep{SuttonBarto}. A major paradigm for solving these problems is value-based RL, in which the agent predicts the \emph{expected return} -- i.e. the expected discounted sum of rewards -- in order to guide its behaviour. The moments or distribution of the random return have also been considered in the literature, with a variety of approaches proposing algorithms for estimating more complex distributional information \citep{ParametricReturnDensityEst,NonparametricReturnDensityEst,ACRiskSensitiveMDPs,VarOfReward}. Recently, \citet{DistPerspective} used the distributional perspective to propose an algorithm, \texttt{C51}, which achieved state-of-the-art performance on the Atari 2600 suite of benchmark tasks. \texttt{C51} is a deep RL algorithm based on categorical policy evaluation (for evaluation) and categorical Q-learning (for control), also introduced by \citet{DistPerspective}, and it is these latter two algorithms which are at the centre of our study. We refer to these approaches as categorical distributional reinforcement learning (CDRL).

Given a state $x$ and action $a$, \texttt{C51} approximates the distribution over returns using a uniform grid over a fixed range, i.e. a \emph{categorical} distribution with evenly-spaced outcomes. Analogous to how value-based approaches such as SARSA \citep{rummery94online} learn to predict, \texttt{C51} also forms a learning target from  sample transitions: reward, next state, and eventually next-state distribution over returns. However, the parallel ends here: because \texttt{C51} learns a distribution, it minimises the Kullback-Leibler divergence between its target and its prediction, rather than the usual squared loss. However, the support of the target is in general disjoint from the approximation support; to account for this, \citet{DistPerspective} further introduced a projection step normally absent from reinforcement learning algorithms.

As a whole, the particular techniques incorporated in \texttt{C51} are not explained by the accompanying theory.
While the ``mean process'' which governs learning within \texttt{C51} is described by a contractive distributional Bellman operator, there are not yet any guarantees on the behaviour of sample-based algorithms. To put things in context, such guarantees in case of estimating expected returns require a completely different mathematical formalism \citep{AsynchSAandQLearning,Jaakkola}.
The effect of the discrete approximation and its corresponding projection step also remain to be quantified. In this paper we analyse these issues.

At the centre of our analysis is the Cram\'er distance between probability distributions. The Cram\'er distance is of particular interest as it was recently shown to possess many of the same properties as the Wasserstein metric, used to show the contractive nature of the distributional Bellman operator \citep{CramerGAN}. Specifically, using the Cram\'er distance, we: (i) quantify the approximation error arising from the discrete approximation in CDRL (see Section \ref{sec:paramandproj}); and (ii) develop stochastic approximation results for the sample-based case (see Section \ref{sec:SA}).

One of the main contributions of this paper is to establish a 
framework for the analysis of CDRL algorithms. This framework reveals a space of possible alternative methods (Sections \ref{sec:c51} and \ref{sec:convergence}). We also demonstrate that the fundamental property required for the convergence of distributional RL algorithms is contractivity of a \emph{projected} Bellman operator, in addition to the contractivity of the Bellman operator itself as in non-distributional RL (Proposition \ref{prop:cat-contract}). This point has parallels with the importance of the (distinct) projection operator in non-tabular RL \citep{LinearTDAnalysis}.

We begin, in Section \ref{sec:background}, with a general introduction to distributional RL, and establish required notation. In Section \ref{sec:c51}, we give a detailed description of categorical distributional RL, and set it in the context of a new framework in which to view distributional RL algorithms. Finally, in Section \ref{sec:convergence}, we undertake a detailed convergence analysis of CDRL, dealing with the approximations and parametrisations that typically must be introduced into practical algorithms. This culminates in the first proofs of convergence for sample-based CDRL algorithms.

\section{BACKGROUND}\label{sec:background}

\subsection{Markov decision processes}
 We consider a Markov decision process (MDP) with a finite state space $\mathcal{X}$, a finite action space $\mathcal{A}$, and a transition kernel $p: \mathcal{X} \times \mathcal{A} \rightarrow \mathscr{P}(\mathbb{R} \times \mathcal{X})$ that defines a joint distribution over immediate reward and next state given a current state-action pair. We will be concerned with stationary policies $\pi : \mathcal{X} \rightarrow \mathscr{P}(\mathcal{A})$ that define a probability distribution over the action space given a current state. The full MDP is given by the collection of random variables $(X_t, A_t, R_t)_{t=0}^\infty$, where $(X_t)_{t \geq 0}$ is the sequence of states taken by the environment, $(A_t)_{t \geq 0}$ is the sequence of actions taken by the agent, and $(R_t)_{t \geq 0}$ is the sequence of rewards.
 
\subsection{Return distributions}
The \emph{return} of a policy $\pi$, starting in initial state $x \in \mathcal{X}$ and initially taking action $a \in \mathcal{A}$, is defined as the random variable given by the sum of discounted rewards:
\begin{align}\label{eq:returns}
\sum_{t=0}^\infty \gamma^t R_t \bigg| X_0 = x, A_0 = a \, ,
\end{align}
where $\gamma \in [0,1)$ is the discount factor. We may implicitly view the distribution of the returns as being parametrised by $\pi$ \citep{PG}. Two common tasks in RL are (i) \emph{evaluation}, in which the expected value of the return is sought for a fixed policy, and (ii) \emph{control}, in which a policy $\pi^*$ which maximises the expected value of the returns is sought.

In the remainder of this paper, we will write
the \emph{distribution} of the return of policy $\pi$ and initial state-action pair $(x, a) \in \mathcal{X} \times \mathcal{A}$ as
\begin{align}\label{eq:returndist}
\eta_\pi^{(x, a)} = \mathrm{Law}_\pi\left( \sum_{t=0}^\infty \gamma^t R_t\ \bigg| X_0 = x, A_0 = a \right) \, .
\end{align}
We write $\eta_\pi$ for the collection of distributions $(\eta_\pi^{(x, a)} | (x, a) \in \mathcal{X} \times \mathcal{A})$.
We highlight the change in emphasis from discussing random variables, as in \eqref{eq:returns}, to directly referring to probability distributions in their own right. 
Although \citet{DistPerspective} referred to the object $\eta_\pi$ as a \emph{value distribution}, here we favour the more technically correct name \emph{return distribution function}, to highlight that $\eta_\pi$ is a function mapping state-action pairs to probability distributions over returns. Referring to return distributions in their own right will lead to a clearer statement of the convergence results that appear in Section \ref{sec:convergence}.

\subsection{The distributional Bellman operator}\label{sec:distbellmanop}
It is well known that expected returns satisfy the Bellman equation \citep{Bellman1957,SuttonBarto}. \citet{DistPerspective} showed that the return distribution function $\eta_\pi$ satisfies a distributional variant of the Bellman equation. This result was phrased in terms of equality in distribution between random variables. A similar approach was taken by \citet{NonparametricReturnDensityEst}, in which cumulative distribution functions were used. To express the Bellman equation in terms of distributions themselves, we will need the notion of pushforward (or image) measures. We first recall the definition of these measures at the level of generality required by the development of our theory; see \citet{Bill86} for further details.

\begin{definition}
Given a probability distribution $\nu \in \mathscr{P}(\mathbb{R})$ and a measurable function $f : \mathbb{R} \rightarrow \mathbb{R}$, the pushforward measure $f_\# \nu \in \mathscr{P}(\mathbb{R})$ is defined by $f_\# \nu(A) = \nu(f^{-1}(A))$, for all Borel sets $A \subseteq \mathbb{R}$.
\end{definition}
 Intuitively, $f_\# \nu$ is obtained from $\nu$ by shifting the support of $\nu$ according to the map $f$. Of particular interest in this paper will be pushforward measures obtained via an affine shift map $f_{r, \gamma} : \mathbb{R} \rightarrow \mathbb{R}$, defined by $f_{r, \gamma}(x) = r + \gamma x$. Such transformations also appear, unnamed, in \citet{ParametricReturnDensityEst}.

Using this notation, we can now restate a fundamental result which was shown by \citet{DistPerspective} in the language of random variables.
The return distribution function $\eta_\pi$ associated with a policy $\pi$, defined in \eqref{eq:returndist}, satisfies the \emph{distributional Bellman equation}:
\[
   \eta_\pi^{(x, a)} = (\mathcal{T}^\pi \eta_\pi)^{(x,a)} \quad  \forall (x, a) \in \mathcal{X} \times \mathcal{A} \, ,
\]
where $\mathcal{T}^\pi : \returndists \rightarrow \returndists$ is the \emph{distributional Bellman operator}, defined by:
\begin{align}\label{eq:distBellmanOp}
& (\mathcal{T}^\pi \eta)^{(x,a)}  \\
 = & \int_{\mathbb{R}}\sum_{(x^\prime, a^\prime) \in \mathcal{X} \times \mathcal{A}} (f_{r, \gamma})_\# \eta^{(x^\prime, a^\prime)} \pi(a^\prime | x^\prime) p(\calcd r, x^\prime | x, a) \nonumber \, ,
\end{align}
for all $\eta \in \returndists$. This equation serves as the basis of distributional RL, just as the standard Bellman equation serves as the basis of non-distributional value-based RL. \citet{DistPerspective} established a preliminary theoretical result regarding the contractive properties of the operator $\mathcal{T}^\pi$. To further this analysis, we first require a particular notion of distance between collections of probability distributions, introduced in \citet{DistPerspective}.

\begin{definition}\label{def:wasserstein}
The $p$-Wasserstein distance $d_p$, for $p \geq 1$ is defined on $\mathscr{P}_p(\mathbb{R})$, the set of probability distributions with finite $p$\textsuperscript{th} moments, by:
\[
d_p(\nu_1, \nu_2) = \left(\inf_{\lambda \in \Lambda(\nu_1, \nu_2)} \int_{\mathbb{R}^2} |x - y|^p \lambda(\calcd x, \calcd y)  \right)^{1/p} \, , 
\]
for all $\nu_1, \nu_2 \in \mathscr{P}_p(\mathbb{R})$, where $\Lambda(\nu_1, \nu_2)$ is the set of probability distributions on $\mathbb{R}^2$ with marginals $\nu_1$ and $\nu_2$.

The supremum-$p$-Wasserstein metric $\overline{d}_p$ is defined on $\mathscr{P}_p(\mathbb{R})^{\mathcal{X} \times \mathcal{A}}$ by
\[
\overline{d}_p(\eta, \mu) = \sup_{(x, a) \in \mathcal{X} \times \mathcal{A}} d_p(\eta^{(x,a)}, \mu^{(x,a)}) \, ,
\]
for all $\eta, \mu \in \mathscr{P}_p(\mathbb{R})^{\mathcal{X} \times \mathcal{A}}$.
\end{definition}

With these definitions in hand, we may recall the following result.

\begin{lemma}[Lemma 3, \cite{DistPerspective}]\label{lem:distcontract}
The distributional Bellman operator $\mathcal{T}^\pi$ is a $\gamma$-contraction in $\overline{d}_p$, for all $p\geq1$. Further, we have, for any initial set of distributions $\eta \in \returndists$:
\[
(\mathcal{T}^\pi)^m \eta \rightarrow \eta_\pi \text{ in } \overline{d}_p , \text{ as } m \rightarrow \infty \, .
\]
\end{lemma}
This motivates \emph{distributional} RL algorithms, which attempt to approximately find $\eta_\pi$
by taking some initial estimates of the return distributions $\eta_0 = (\eta_0^{(x, a)} | (x, a) \in \mathcal{X}\times\mathcal{A})$, and iteratively computing a sequence of estimates $(\eta_t)_{t \geq0}$ by approximating the update step
\begin{align}\label{eq:idealDistBellmanBackup}
\eta_{t+1} \leftarrow \mathcal{T}^\pi \eta_t  \text{\ \ \ for\ } t=0,1,\ldots\ \, .
\end{align}
There is also a control version of these updates, which seeks to find the return distributions associated with an optimal policy $\pi^*$, via the following updates
\begin{align}\label{eq:idealDistControl}
\eta_{t+1} \leftarrow \mathcal{T} \eta_t  \text{\ \ \ for\ } t=0,1,\ldots\ \, ,
\end{align}
where $\mathcal{T}$ is the control version of the distributional Bellman operator, defined by
\begin{align*}
 (\mathcal{T} \eta)^{(x,a)}\!
 = &\! \int_{\mathbb{R}}\!\sum_{(x^\prime, a^\prime) \in \mathcal{X} \times \mathcal{A}}\!\!\! (f_{r, \gamma})_\# \eta^{(x^\prime, a^*(x^\prime))} p(\calcd r, x^\prime | x, a) \nonumber \, , \\
 & \quad\text{where\ } a^*(x^\prime) \in \argmax_{a^\prime \in \mathcal{A}} \mathbb{E}_{R \sim \eta^{(x^\prime, a^\prime)}}\left\lbrack R \right\rbrack \, .
\end{align*}
An ideal policy evaluation algorithm would iteratively compute the exact updates of \eqref{eq:idealDistBellmanBackup}, and inherit the resulting convergence guarantees from Lemma \ref{lem:distcontract}. However, full computation of the distributional Bellman operator on a return distribution function is typically either impossible (due to unknown MDP dynamics), or computationally infeasible \citep{NDP}. In order to take the full updates in \eqref{eq:idealDistBellmanBackup} or \eqref{eq:idealDistControl} and produce a practical, scalable distributional RL algorithm, several key approximations are required, namely:
\begin{enumerate}[label=(\roman*),nosep]
	\item \emph{distribution parametrisation};
	\item \emph{stochastic approximation} of the Bellman operator;
	\item \emph{projection} of the Bellman target distribution; 
	\item \emph{gradient updates} via a loss function.
\end{enumerate}
We discuss each of these approximations in Section \ref{sec:c51}, at the same time describing our two CDRL algorithms, categorical policy evaluation and categorical Q-learning, in detail with this approximation framework in mind.

\section{CATEGORICAL POLICY EVALUATION AND CATEGORICAL Q-LEARNING}\label{sec:c51}

\begin{figure*}[ht]
\begin{center}
\includegraphics[width=\textwidth]{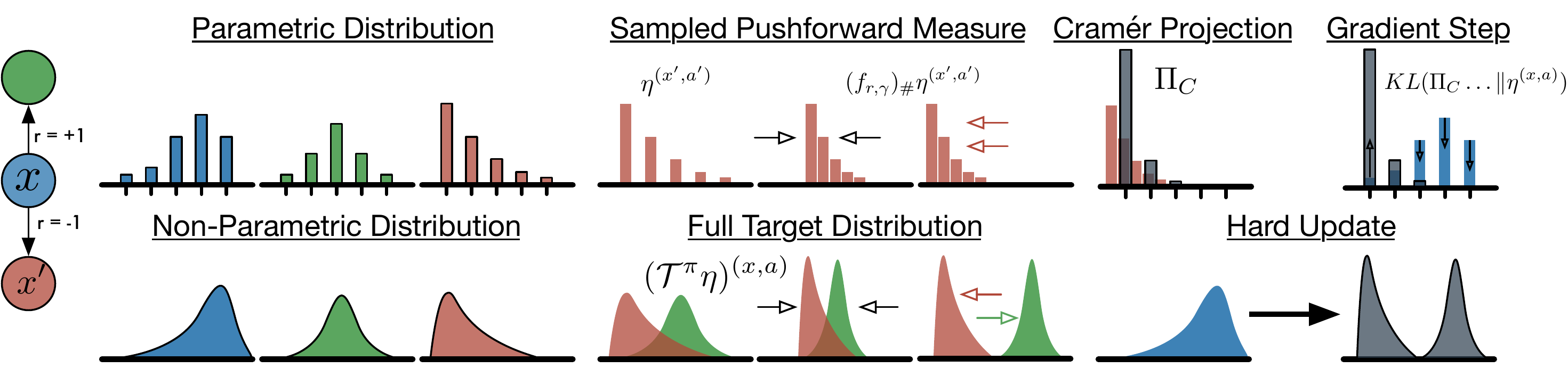}
\end{center}
\caption{A 3-state MDP with a single action available at each state (shown far left), with full update scheme \eqref{eq:idealDistBellmanBackup} illustrated on the bottom row, and the categorical policy evaluation update scheme illustrated on the top row. For both update schemes, the current return distribution function is illustrated on the left, the computation of the target distribution in the centre, and finally the update is shown on the right.\label{fig:c51}}
\end{figure*}

Our first contribution is to make explicit the various approximations, parametrisations, and assumptions implicit in CDRL algorithms.
Categorical poicy evaluation approximates the update scheme \eqref{eq:idealDistBellmanBackup}; it produces an iterative sequence $(\eta_t)_{t \geq 0}$ of approximate return distribution functions, updating the approximations as shown in Algorithm \ref{alg:c51}. Figure \ref{fig:c51} illustrates the salient points of the algorithm, and contrasts them against the full updates of \eqref{eq:idealDistBellmanBackup}. Algorithm \ref{alg:c51} also describes categorical Q-learning, which approximates the full updates in \eqref{eq:idealDistControl}. We now discuss the structure of Algorithm \ref{alg:c51} in more detail, with reference to the distributional RL framework introduced at the end of Section \ref{sec:distbellmanop}.

\begin{algorithm}
	\caption{CDRL update \citep{DistPerspective}}\label{alg:c51}
	\begin{algorithmic}[1]
		\REQUIRE $\eta_t^{(x, a)} = \sum_{k=1}^K p^{(x, a)}_{t, k} \delta_{z_k}$ for each $(x, a)$
		\STATE Sample transition $(x_t, a_t, r_t, x_{t+1})$ \label{algoline:SAstart}
		\STATE \textcolor{gray}{\# Compute distributional Bellman target}
		\IF{Categorical policy evaluation}
		\STATE $a^* \sim \pi(\cdot | x_{t+1})$
		\ELSIF{Categorical Q-learning}
		\STATE $a^* \leftarrow \argmax_a \mathbb{E}_{R \sim \eta_t^{(x_{t+1}, a)}}[R]$
		\ENDIF
		\STATE $\widehat{\eta}^{(x_t, a_t)}_* \leftarrow (f_{r_t, \gamma})_\# \eta_t^{(x_{t+1}, a^*)}$ \label{algoline:SAend}
		\STATE \textcolor{gray}{\# Project target onto support}
		\STATE $\widehat{\eta}^{(x_t, a_t)}_t \leftarrow \Pi_\mathcal{C} \widehat{\eta}_*^{(x_t, a_t)}$ \label{algoline:projection}
		\STATE \textcolor{gray}{\# Compute KL Loss}
		\STATE Find gradient of $\text{KL}(\widehat{\eta}^{(x_t, a_t)}_t || \eta_t^{(x_t, a_t)})$ \label{algoline:klgrad}
		\STATE Use gradient to generate new estimate ${\eta_{t+1}^{(x_t, a_t)} = \sum_{k=1}^K p^{(x_t, a_t)}_{t+1, k} \delta_{z_k}}$ \label{algoline:dist1}
		\RETURN $\eta_{t+1}^{(x, a)} = \sum_{k=1}^K p^{(x, a)}_{t+1, k} \delta_{z_k}$ for each $(x, a)$ \label{algoline:dist2}
	\end{algorithmic}
\end{algorithm}

\subsection{Distribution parametrisation}\label{sec:distributionparam}
From an algorithmic perspective, it is impossible to represent the full space of probability distributions $\mathscr{P}(\mathbb{R})$ with a finite collection of parameters.
Therefore a first design decision for a general distributional RL algorithm is how probability distributions should be represented in an approximate way.
Formally, this requires the selection of a parametric family $\mathcal{P} \subset \mathscr{P}(\mathbb{R})$. CDRL uses the parametric family
\[
\mathcal{P} = \left\{ \sum_{i=1}^K p_i \delta_{z_i} \bigg| p_1, \ldots, p_K \geq 0 \, , \sum_{k=1}^K p_k = 1 \right\} \, ,
\]
of \emph{categorical} distributions over some fixed set of equally-spaced supports $z_1 < \cdots < z_K$); see lines \ref{algoline:dist1} and \ref{algoline:dist2} of Algorithm \ref{alg:c51}. Other parametrisations are of course possible, such as mixtures of Diracs with varying location parameters \citep{QR}, mixtures of Gaussians, etc.

\subsection{Stochastic approximation of Bellman operator}\label{sec:SAOfBellman}
Evaluation of the distributional Bellman operator $\mathcal{T}^\pi$ (see \eqref{eq:distBellmanOp}) requires integrating over all possible next state-action-reward combinations. Some approximation is required; a popular way to achieve this in RL is by \emph{sampling} a transition $(x_t, a_t, r_t, x_{t+1}, a^*)$ of the MDP.
This is also the approach taken in CDRL, as shown in lines \ref{algoline:SAstart}-\ref{algoline:SAend} of Algorithm \ref{alg:c51}. Here $a^*$ is selected either by sampling from the policy $\pi(\cdot|x_{t+1})$ in the case of categorical policy evaluation, or as the action with the highest estimated expected returns, in the case of categorical Q-learning. In the context of categorical policy evaluation, this defines a stochastic Bellman operator, given by
\begin{align}\label{eq:stocBellman}
(\widehat{\mathcal{T}}^\pi \eta_t)^{(x_t, a_t)} &= (f_{r_t, \gamma})_\# \eta_t^{(x_{t+1}, a^*)} \, , \\
(\widehat{\mathcal{T}}^\pi \eta_t)^{(x, a)} &= \eta_t^{(x, a)} \qquad \text{\ if\ } (x, a) \not= (x_t, a_t) \, , \nonumber
\end{align}
where the randomness in $\widehat{\mathcal{T}}^\pi$ comes from the randomly sampled transition $(x_t, a_t, r_t, x_{t+1}, a^*)$. Note that this defines a \emph{random measure}, and importantly, this random measure is equal \emph{in expectation} to the true Bellman target $(\mathcal{T}^\pi \eta_t)^{(x_t, a_t)}$.

\subsection{Projection of Bellman target distribution}\label{sec:proj}
Having computed $(\widehat{\mathcal{T}}^\pi \eta_t)^{(x_t, a_t)}$, this new distribution typically no longer lies in the parametric family $\mathcal{P}$; as shown in \eqref{eq:stocBellman}, the supports of the distributions are transformed by an affine map $f_{r, \gamma}$. We therefore require a method of mapping the backup distribution function into the parametric family. That is, we require a \emph{projection operator} $\Pi : \mathscr{P}(\mathbb{R}) \rightarrow \mathcal{P}$ that may be applied to each real-valued distribution in a return distribution function. CDRL uses the heuristic projection operator $\Pi_\mathcal{C}$ (see line \ref{algoline:projection} of Algorithm \ref{alg:c51}), which was defined by \citet{DistPerspective} as follows for single Dirac measures:
\begin{align}\label{eq:cramerprojdiracs}
\Pi_\mathcal{C} (\delta_{y}) = \begin{cases}
\delta_{z_1} &  y \leq z_1 \\
\frac{z_{i+1} - y}{z_{i+1} - z_i} \delta_{z_i} + \frac{y - z_i}{z_{i+1} - z_i} \delta_{z_{i+1}} &  z_i < y \leq z_{i+1} \\
\delta_{z_K} &  y > z_K
\end{cases}
\, ,
\end{align}
and extended affinely to finite mixtures of Dirac measures, so that for a mixture of Diracs $\sum_{i=1}^N p_i \delta_{y_i}$, we have $\Pi_\mathcal{C}(\sum_{i=1}^N p_i \delta_{y_i}) = \sum_{i=1}^N p_i \Pi_\mathcal{C} (\delta_{y_i})$ - see the right-hand side of Figure \ref{fig:c51}. In general we will abuse notation, and use $\Pi_\mathcal{C}$ to denote the projection operator for individual distributions, and also the operator on return distribution functions $\mathscr{P}(\mathbb{R})^{\mathcal{X} \times \mathcal{A}} \rightarrow \mathcal{P}^{\mathcal{X} \times \mathcal{A}}$, which applies the former projection to each distribution in the return distribution function.

\subsection{Gradient updates}\label{sec:grad}
Having computed a stochastic approximation $\widehat{\eta}_t^{(x_t, a_t)} = (\Pi_\mathcal{C} \widehat{\mathcal{T}}^\pi \eta_t)^{(x_t, a_t)}$ to the full target distribution, the remaining issue is how the next iterate $\eta_{t+1}$ should be defined. In \texttt{C51}, the approach is to perform a single step of gradient descent on the Kullback-Leibler divergence of the prediction $\eta_t^{(x_t, a_t)}$ from the target $\widehat{\eta}_t^{(x_t, a_t)}$:
\begin{align}\label{eq:klobjective}
\text{KL}(\widehat{\eta}^{(x_t, a_t)}_t || \eta_t^{(x_t, a_t)}) \, ,
\end{align}
with respect to the parameters of $\eta_t^{(x_t, a_t)}$ - see line \ref{algoline:klgrad} of Algorithm \ref{alg:c51}. We also consider CDRL algorithms based on a mixture update, described in more detail in Section \ref{sec:SA}. The use of a gradient update, rather than a ``hard'' update allows for the dissipation of noise introduced in the target by stochastic approximation \citep{NDP,SAandRAandA}.
This completes the description of CDRL in the context of the framework introduced at the end of Section \ref{sec:distbellmanop}; we now move on to discussing the convergence properties of these algorithms.

\section{CONVERGENCE ANALYSIS}\label{sec:convergence}

The approximations, parametrisations, and heuristics of CDRL discussed in Section \ref{sec:c51} yield practical, scalable algorithms for evaluation and control, but the effects of these heuristics on the theoretical guarantees that many non-distributional algorithms enjoy have not yet been addressed. In this section, we set out a variety of theoretical results for CDRL algorithms, and in doing so, emphasise several key ways in which the approximations described in Section \ref{sec:c51} must fit together to enjoy good theoretical guarantees.

We begin by drawing a connection between the heuristic projection operator $\Pi_\mathcal{C}$ and the Cram\'er distance in Section \ref{sec:cramergeometry}. This connection then paves the way to obtaining the results of Section \ref{sec:paramandproj}, which concern the properties of CDRL policy evaluation algorithms without stochastic approximation and gradient updates, observing only the consequences of the parametrisation and projection steps discussed in Sections \ref{sec:distributionparam} and \ref{sec:proj}. We then bring these more realistic assumptions into play in Section \ref{sec:SA}, and our analysis culminates in a proof of convergence of categorical policy evaluation and categorical Q-learning in the tabular setting.

\subsection{Cram\'er geometry}\label{sec:cramergeometry}

We begin by recalling Lemma \ref{lem:distcontract}, through which \citet{DistPerspective} established that repeated application of the distributional Bellman operator $\mathcal{T}^\pi$ to an initial return distribution function guarantees convergence to the true set of return distributions in the supremum-Wasserstein metric.
However, once we introduce the parametrisation $\mathcal{P}$ and projection operator $\Pi_\mathcal{C}$ of categorica policy evaluation, the operator of concern is now $\Pi_\mathcal{C} \mathcal{T}^\pi$, the composition of the Bellman operator $\mathcal{T}^\pi$ with the projection operator $\Pi_\mathcal{C}$. Our first result illustrates that the presence of the projection operator is enough to break the contractivity under Wasserstein distances.

\begin{restatable}{lemma}{CramProjNotWNonExpansion}\label{lem:CramProjNotWNonExpansion}
The operator $\Pi_\mathcal{C}\mathcal{T}^\pi$ is in general not a contraction in $\overline{d}_p$, for $p > 1$.
\end{restatable}

Whilst contractivity with respect to $\overline{d}_1$ is in fact maintained, as we shall see there is a much more natural metric, the Cram\'er distance \citep{CramerDist}, with which to establish contractivity of the combined operator $\Pi_\mathcal{C} \mathcal{T}^\pi$.

\begin{definition}\label{def:cramer}
The Cram\'er distance $\ell_2$ between two distributions $\nu_1, \nu_2 \in \mathscr{P}(\mathbb{R})$, with cumulative distribution functions $F_{\nu_1}, F_{\nu_2}$ respectively, is defined by:
\[
\ell_2(\nu_1, \nu_2) = \left( \int_\mathbb{R} (F_{\nu_1}(x) - F_{\nu_2}(x))^2 \calcd x \right)^{1/2} \, .
\]
Further, the supremum-Cram\'er metric $\overline{\ell}_2$ is defined between two distribution functions $ \eta, \mu \in \mathscr{P}(\mathbb{R})^{\mathcal{X} \times \mathcal{A}}$ by
\[
\overline{\ell}_2(\eta, \mu) = \sup_{(x, a) \in \mathcal{X} \times \mathcal{A}} \ell_2(\eta^{(x,a)}, \mu^{(x,a)}) \, .
\]
\end{definition}

The Cram\'er distance was recently studied as an alternative to the Wasserstein distances in the context of generative modelling \citep{CramerGAN}. The Cram\'er distance, in fact, induces a useful geometric structure on the space of probability measures. We use this to provide a new interpretation of the heuristic projection $\Pi_\mathcal{C}$ intimately connected with the Cram\'er distance. The salient points of this connection are stated in Proposition \ref{prop:projnonexpansion}, with full mathematical details provided in the corresponding proof in the appendix. We then use this in Section \ref{sec:paramandproj} to show that $\Pi_\mathcal{C} \mathcal{T}^\pi$ is a contraction in $\overline{\ell}_2$.

\begin{restatable}{proposition}{projnonexpansion}\label{prop:projnonexpansion}
The Cram\'er metric $\ell_2$ endows a particular subset of $\mathscr{P}(\mathbb{R})$ with a notion of orthogonal projection, and the orthogonal projection onto the subset $\mathcal{P}$ is exactly the heuristic projection $\Pi_\mathcal{C}$.  Consequently, $\Pi_\mathcal{C}$ is a non-expansion with respect to $\ell_2$.
\end{restatable}

\begin{figure}
	\centering
	\includegraphics[keepaspectratio, width=0.42\textwidth]{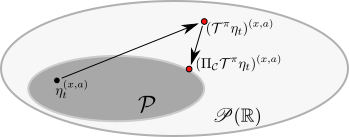}
	\caption{An illustration of the composition of the distributional Bellman operator with the projection $\Pi_\mathcal{C}$, interpreting probability distributions as points in an affine Hilbert space.}
	\label{fig:hilbertspaceproj}
\end{figure}

A consequence of the result above is the following, which will be useful in later sections.

\begin{restatable}[Pythagorean theorem]{lemma}{pythag}\label{lem:pythag}
Let $\mu \in \mathscr{P}([z_1, z_K])$, and let $\nu \in \mathscr{P}(\{z_1, \ldots, z_K\})$. Then
\[
\ell_2^2(\mu, \nu) = \ell_2^2(\mu, \Pi_\mathcal{C} \mu) + \ell_2^2(\Pi_\mathcal{C} \mu, \nu) \, .
\]
\end{restatable}

A geometric illustration of the action of the composed operator $\Pi_\mathcal{C} \mathcal{T}^\pi$ is given in Figure \ref{fig:hilbertspaceproj}, in light of the interpretation of $\Pi_\mathcal{C}$ as an orthogonal projection.

\subsection{Parametrisation and projection}\label{sec:paramandproj}
Having established these tools, we can now prove contractivity of the operator $\Pi_\mathcal{C} \mathcal{T}^\pi$, and hence convergence of this variant of distributional RL in the absence of stochastic approximation.

\begin{restatable}{proposition}{catcontract}\label{prop:cat-contract}
The operator $\Pi_\mathcal{C}\mathcal{T}^\pi$ is a $\sqrt{\gamma}$-contraction in $\overline{\ell}_2$. Further, there is a unique distribution function $\eta_\mathcal{C} \in \mathcal{P}^{\mathcal{X} \times \mathcal{A}}$ such that given any initial distribution function $\eta_0 \in \returndists$, we have
\[
(\Pi_\mathcal{C} \mathcal{T}^\pi)^m\eta_0 \rightarrow \eta_\mathcal{C} \text{\ in \ } \overline{\ell}_2 \text{ as } m \rightarrow \infty \, .
\]
\end{restatable}

A natural question to ask is how the limiting distribution function $\eta_\mathcal{C}$, established in Proposition \ref{prop:cat-contract}, differs from the true distribution function $\eta_\pi$. In some sense, this quantifies the ``cost'' of using the parametrisation $\mathcal{P}$ rather than learning fully non-parametric probability distributions. Reusing the interpretation of $\Pi_\mathcal{C}$ as an orthogonal projection, and using a geometric series argument, we may establish the following result, which echoes existing results for linear function approximation \citep{LinearTDAnalysis}.

\begin{restatable}{proposition}{cramerdistetac}\label{prop:cramerdist-etaC}
Let $\eta_\mathcal{C}$ be the limiting return distribution function of Proposition \ref{prop:cat-contract}. If $\eta^{(x, a)}_\pi$ is supported on $[z_1, z_K]$ for all $(x, a) \in \mathcal{X} \times \mathcal{A}$, then we have:
\[
\overline{\ell}_2^2(\eta_\mathcal{C}, \eta_\pi)  \leq \frac{1}{1 - \gamma} \max_{1 \leq i < K} (z_{i+1} - z_i)  \, .
\]
\end{restatable}

This establishes that as the fineness of the grid $\{z_1, \ldots, z_K\}$ increases, we gradually recover the true return distribution function. The bound in Proposition \ref{prop:cramerdist-etaC} relies on a guarantee that the support of the true return distributions lie in the interval $[z_1, z_K]$. Many RL problems come with such a guarantee, but there are also many circumstances where \emph{a priori} knowledge of the scale of rewards is unavailable. It is possible to modify the proof of Proposition \ref{prop:cramerdist-etaC} to deal with this situation too.

\begin{restatable}{proposition}{cramerdistmisspecified}\label{prop:cramerdist-etaC_misspecified}
Let $\eta_\mathcal{C}$ be the limiting return distribution function of Proposition \ref{prop:cat-contract}. Suppose $\eta^{(x, a)}_\pi$ is supported on an interval $[z_1 - \delta, z_K + \delta]$ containing $[z_1, z_K]$ for each $(x, a) \in \mathcal{X} \times \mathcal{A}$, and $\eta^{(x, a)}_\pi( [z_1-\delta, z_1] \cup [z_K, z_K + \delta]) \leq q$ for some $q \in \mathbb{R}$ and for all $(x, a) \in \mathcal{X} \times \mathcal{A}$ -- $q$ bounds the excess mass lying outside the region $[z_1, z_K]$. Then we have
\begin{align*}
\overline{\ell}_2^2(\eta_\mathcal{C}, \eta_\pi)
\leq \frac{1}{1 - \gamma}\left( \max_{1 \leq i < K} (z_{i+1} - z_i) + 2q^2\delta \right) \, .
\end{align*}
\end{restatable}

\subsection{Stochastic approximation and gradient updates}\label{sec:SA}
In this section, we leverage the theory of stochastic approximation to provide convergence guarantees for sample-based distributional RL algorithms.

We will study a version of categorical policy evaluation that takes a mixture between two distributions, rather than using a KL gradient, as a means of updating the return distribution estimates. The algorithm proceeds by computing the target distribution $\widehat{\eta}^{(x_t, a_t)}_t$ as in Algorithm \ref{alg:c51}, but then rather than using the gradient of a KL loss, the updated return distribution is produced for some collection of learning rates $(\alpha_t(x, a) | (x, a) \in \mathcal{X} \times \mathcal{A}, t \geq 0)$ according to the following rule:
\begin{align}\label{eq:mixtureupdate}
&\eta_{t+1}^{(x, a)} \leftarrow (1-\alpha_t(x, a)) \eta_t^{(x, a)} + \alpha_t(x,a) \widehat{\eta}^{(x, a)}_t \ \forall (x, a) , \nonumber \\
&\mathrm{such\ that\ } \alpha_t(x, a) = 0 \mathrm{\ if\ } (x, a) \not= (x_t, a_t) \, .
\end{align}
That is, by taking a \emph{mixture} between $\eta_t^{(x_t, a_t)}$ and $\widehat{\eta}^{(x_t, a_t)}_t$. We denote this procedure as Algorithm \ref{alg:mixture}, which for completeness is stated in full in Section \ref{sec:mixturealgo} of the appendix. The question of whether convergence results hold for the KL update described in Section \ref{sec:grad} remains open, and is an interesting area for further research.

\subsubsection{Convergence of categorical policy evaluation}
We first show that, under standard conditions, categorical policy evaluation with the mixture update rule described above is guaranteed to converge to the fixed point of the  projected Bellman operator $\Pi_\mathcal{C} \mathcal{T}^\pi$, as described in Proposition \ref{prop:cat-contract}. We sketch out the main structure of the proof below; the full argument is given in the appendix.

\begin{restatable}{theorem}{SAevaluation}\label{thm:SAevaluation}
In the context of policy evaluation for some policy $\pi$, suppose that:
\begin{enumerate}[label=(\roman*)]
    \item \label{cond:RM} the stepsizes $(\alpha_t(x ,a)|t \geq 0, (x, a) \in \mathcal{X} \times \mathcal{A})$ satisfy the Robbins-Monro conditions:
    \begin{itemize}
        \item $\sum_{t=0}^\infty \alpha_t(x, a) = \infty$
        \item $\sum_{t=0}^\infty \alpha^2_t(x, a) <  C <\infty$
    \end{itemize}
    almost surely, for all $(x, a) \in \mathcal{X} \times \mathcal{A}$;
    \item \label{cond:boundedinitial} we have initial estimates $\eta_0^{(x, a)}$ of the distribution of returns for each state-action pair $(x, a) \in \mathcal{X} \times \mathcal{A}$, each with support contained in $[z_1, z_K]$.
\end{enumerate}
Then, for the updates given by Algorithm \ref{alg:mixture}, in the case of evaluation of the policy $\pi$, we have almost sure convergence of $\eta_t$ to $\eta_\mathcal{C}$ in $\overline{\ell}_2$, where $\eta_{\mathcal{C}}$ is the limiting return distribution function of Proposition \ref{prop:cat-contract}. That is,
\[
\overline{\ell}_2(\eta_t, \eta_\mathcal{C}) \rightarrow 0 \text{\ as\ } t \rightarrow \infty \text{\ \ almost surely.}
\]
\end{restatable}

The proof follows the approach of \citet{AsynchSAandQLearning}; we combine classical stochastic approximation proof techniques with notions of stochastic dominance to prove the almost-sure convergence of the return distribution functions in $\overline{\ell}_2$. Proposition \ref{prop:monotonemaps} is an interesting result in its own right, as it establishes a formal language to describe the monotonocity of the distributional Bellman operator, which plays an important role in control operators \citep[e.g.][]{bertsekas12dynamic}.

We begin by showing that several variants of the Bellman operator are \emph{monotone} with respect to a particular partial ordering over probability distributions known as stochastic dominance \citep{StochasticOrders}.

\begin{definition}
Given two probability measures $\nu_1, \nu_2 \in \mathscr{P}(\mathbb{R})$, we say that $\nu_1$ \emph{stochastically dominates} $\nu_2$, and write $\nu_2 \leq \nu_1$, if there exists a coupling between $\nu_1$ and $\nu_2$ (that is, a probability measure on $\mathbb{R}^2$ with marginals given by $\nu_1$ and $\nu_2$) which is supported on the set $\{ (x_1, x_2) \in \mathbb{R}^2 | x_2 \geq x_1 \}$. An equivalent characterisation states that $\nu_2 \leq \nu_1$ if for the corresponding CDFs $F_{\nu_1}$ and $F_{\nu_2}$, we have
\[
F_{\nu_2}(x) \geq F_{\nu_1}(x) \quad \text{\ for\ all\ } x \in \mathbb{R} \, .
\]
\end{definition}
Stochastic dominance forms a partial order over the set $\mathscr{P}(\mathbb{R})$. We introduce a related partial order over the space of return distribution functions, $\returndists$, which we refer to as (element-wise) stochastic dominance. Given $\eta, \mu \in \returndists$, we say that $\eta$ stochastically dominates $\mu$ element-wise if for each $(x, a) \in \mathcal{X} \times \mathcal{A}$, $\eta^{(x, a)}$ stochastically dominates $\mu^{(x, a)}$.
\begin{restatable}{proposition}{monotonebellman}\label{prop:monotonemaps}
The distributional Bellman operator $\mathcal{T}^\pi : \returndists \rightarrow \returndists$ is a monotone map with respect to the partial ordering on $\returndists$ given by element-wise stochastic dominance. Further,
the Cram\'er projection $\Pi_\mathcal{C} : \returndists \rightarrow \returndists$ is a monotone map, from which it follows that the Cram\'er-Bellman operator $\Pi_\mathcal{C} \mathcal{T}^\pi$ is also monotone.
\end{restatable}

The monotonicity of the mappings described in Proposition \ref{prop:monotonemaps} can then be harnessed to establish a chain of lemmas, given in the appendix, mirroring the chain of reasoning in \cite{AsynchSAandQLearning}, from which Theorem \ref{thm:SAevaluation} will follow. In the remainder of this section, we highlight a further important property of the Cram\'er projection $\Pi_\mathcal{C}$ which is crucial in establishing Theorem \ref{thm:SAevaluation}.

We observe from Algorithm \ref{alg:mixture} that the update rule appearing in Equation \eqref{eq:mixtureupdate} can be written
\begin{align*}
\eta_{t+1}^{(x, a)} =&  \eta_t^{(x, a)} + \alpha_t(x, a) ((\Pi_\mathcal{C} \mathcal{T}^\pi \eta_t)^{(x, a)} - \eta_t^{(x, a)}) \\ &+ \alpha_t(x,a) (\Pi_\mathcal{C}(f_{r, \gamma})_\# \eta_t^{(x^\prime, a^\prime)} - (\Pi_\mathcal{C} \mathcal{T}^\pi \eta_t)^{(x, a)}) \, .
\end{align*}
for all $(x, a) \in \mathcal{X} \times \mathcal{A}$, for all $t \geq 0$, given that $\alpha_t(x, a) = 0$ if the state $(x, a) \in \mathcal{X} \times \mathcal{A}$ is not selected for update at time $t$.
The second term,
\[
\alpha_t(x, a) ((\Pi_\mathcal{C} \mathcal{T}^\pi \eta_t)^{(x, a)} - \eta_t^{(x, a)}) \, ,
\]
may be interpreted as a damped version of the full distributional Bellman update, whilst the third term,
\[\alpha_t(x,a) (\Pi_\mathcal{C}(f_{r, \gamma})_\# \eta_t^{(x^\prime, a^\prime)} - (\Pi_\mathcal{C} \mathcal{T}^\pi \eta_t)^{(x, a)}) \, ,
\]
represents the noise introduced by stochastic approximation. We observe that this noise term is in fact a difference of two probability distributions (one of which is a random measure); thus, this noise term is a particular instance of a random \emph{signed measure}. The Cram\'er projection leads to an important property of this signed measure, which is crucial in establishing the result of Theorem \ref{thm:SAevaluation}, summarised in Lemma \ref{lem:signedmeasure}. 

\begin{lemma}\label{lem:signedmeasure}
The noise term
\[
\Pi_\mathcal{C}(f_{r, \gamma})_\# \eta_t^{(x^\prime, a^\prime)} - (\Pi_\mathcal{C} \mathcal{T}^\pi \eta_t)^{(x, a)}
\]
is a random signed measure with total mass $0$ almost surely, and with the property that when averaged over the next-step reward, state and action tuple $(r, x^\prime, a^\prime)$ it is equal to the zero measure almost surely:
\begin{align*}
&\mathbb{E}_{r, x^\prime, a^\prime}\left\lbrack(\Pi_\mathcal{C}(f_{r, \gamma})_\# \eta_t^{(x^\prime, a^\prime)} - (\Pi_\mathcal{C} \mathcal{T}^\pi \eta_t)^{(x, a)})\right\rbrack((-\infty, y]) \\
&\qquad\qquad\qquad\qquad\qquad\qquad\qquad\qquad\qquad\qquad= 0 \, ,
\end{align*}
for all $y \in \mathbb{R}$.
\end{lemma}

\subsubsection{Convergence of categorical Q-learning}

Having established convergence of categorical policy evaluation in Theorem \ref{thm:SAevaluation}, we now leverage this to prove convergence of categorical Q-learning under similar conditions.

\begin{restatable}{theorem}{SAcontrol}\label{thm:SAcontrol}
Suppose that Assumptions \ref{cond:RM}--\ref{cond:boundedinitial} of Theorem \ref{thm:SAevaluation} hold, and that all unprojected target distributions $\widehat{\eta}_*^{(x_t, a_t)}$ arising in Algorithm \ref{alg:mixture} are supported within $[z_1, z_K]$ almost surely.
Assume further that there is a unique optimal policy $\pi^*$ for the MDP. Then, for the updates given in Algorithm \ref{alg:mixture}, in the case of control, we have almost sure convergence of $(\eta_t^{(x, a)})_{(x, a) \in \mathcal{X} \times \mathcal{A}}$ in $\overline{\ell}_2$ to some limit $\eta_\mathcal{C}^*$, and furthermore the greedy policy with respect to $\eta_\mathcal{C}^*$ is the optimal policy $\pi^*$.
\end{restatable}
Theorem \ref{thm:SAcontrol} is particularly interesting because it demonstrates that value-based control is not only stable in the distributional case, but also that CDRL \emph{preserves the optimal policy}. This is not a given: for example, if we were to replace $\Pi_{\mathcal{C}}$ with a nearest-neighbour-type projection we could not provide the same guarantee. What makes the CDRL projection step special in this regard is that it preserves the expected value of the unprojected target.

\section{DISCUSSION}
The $\texttt{C51}$ algorithm was empirically successful, but, as we have seen in Lemma \ref{lem:CramProjNotWNonExpansion}, is not explained by the initial theoretical results concerning CDRL of \citet{DistPerspective}. We have now shown that the projected distributional Bellman operator used in CDRL inherits convergence guarantees from a different metric altogether, the Cram\'er distance. From Propositions \ref{prop:cramerdist-etaC} and \ref{prop:cramerdist-etaC_misspecified}, we see that the limiting approximation error is controlled by the granularity of the parametric distribution and the discount factor $\gamma$. Furthermore, we have shown that in the stochastic approximation setting this update converges both for policy evaluation and control.

An important aspect of our analysis is the role of the projection onto the set of parametrised distributions, in distributional RL. Just as existing work has studied the role of the projected Bellman operator in function approximation \citep{LinearTDAnalysis}, there is a corresponding importance for considering the effects of the projection in distributional RL.

\subsection{Function approximation}\label{sec:functionapproximation}
Our theoretical results in Section \ref{sec:convergence} treat the problem of tabular distributional RL, with an approximate parametrisation distribution for each state-action pair.
Theoretical understanding of function approximation in RL has been the focus of much research, and has significantly improved our understanding of agent behaviour. Although we believe the effects of function approximation on distributional RL are of great theoretical and empirical interest, we leave the function approximation setting as an interesting direction for future work.

\subsection{Theoretically grounded algorithms}\label{sec:groundedalgs}
Turning theoretical results into practical algorithms can often be quite challenging. However, our results do suggest some immediate directions for potential improvements to $\texttt{C51}$. First, the convergence results for stochastic approximation suggest that an improved algorithm could be obtained by either directly minimising the Cram\'er distance or through a regularised KL minimisation that more closely reflects the mixture updates in Section \ref{sec:SA}. Second, the results of Propositions \ref{prop:cramerdist-etaC} and \ref{prop:cramerdist-etaC_misspecified} indicate that if our support is densely focused around the true range of returns we should expect significantly better performance, due to the effects of the discount factor. Improving this by either prior domain knowledge or adapting the support to reflect the true return range could yield much better empirical performance.

\section{CONCLUSION}
In this paper we have introduced a framework for distributional RL algorithms, and provided convergence analysis of recently proposed algorithms. We have introduced the notion of the projected distributional Bellman operator and argued for its importance in the theory of distributional RL.

Interesting future directions from an empirical perspective include exploring the space of possible distributional RL algorithms set out in Section \ref{sec:c51}. From a theoretical perspective, the issue of how function approximation interacts with distributional RL remains an important open question.

\newpage

\section*{Acknowledgements}
The authors acknowledge the important contributions of their colleagues at DeepMind. Special thanks to Wojciech Czarnecki, Chris Maddison, Ian Osband and Grzegorz Swirszcz for their early suggestions and discussions. Thanks also to Clare Lyle for useful comments.
\bibliography{aistatsbib}

\onecolumn
\newpage

\section*{Appendix}

\section{A supporting result}\label{sec:support}
We first present an alternative characterisation of the projection operator $\Pi_\mathcal{C}$ which will be useful for the analysis that follows. Throughout, for a probability measure $\nu \in \mathscr{P}(\mathbb{R})$, we write $F_\nu$ for its CDF.
\begin{restatable}{proposition}{characteriseproj}\label{prop:characteriseproj}
For each $i=1,\ldots, K$, define $h_{z_i} : \mathbb{R} \rightarrow [0,1]$ to be the (possibly asymmetric) hat function centered in $z_i$ defined by
\[
h_{z_i}(x)= 
\left\{
\begin{array}{lll}
 \frac{z_{i+1}-x}{z_{i+1}-z_i} & \mbox{ for } x\in [z_i, z_{i+1}] &\mbox{ and } 1\leq i < K,\\
 \frac{x-z_{i-1}}{z_{i}-z_{i-1}} & \mbox{ for } x\in [z_{i-1}, z_{i}] &\mbox{ and } 1< i \leq K, \\
 1  & \mbox{ for } x \leq z_1 &\mbox{ and } i=1, \\
 1  & \mbox{ for } x\geq z_K &\mbox{ and } i=K, \\
 0 & \mbox{ otherwise.} &
\end{array}
\right.
\]
Then defining $\Pi_\mathcal{C} \nu = \sum_{i=1}^K \mathbb{E}_{w \sim\nu}[h_{z_i}(w)] \delta_{z_i}$ for all probability distributions $\nu \in \mathscr{P}(\mathbb{R})$, is consistent with the earlier definition in \eqref{eq:cramerprojdiracs} for mixtures of Diracs. Further, $F_{\Pi_\mathcal{C} \nu}(z_i)$ is equal to the average value of $F_{\nu}$ in the interval $[z_i, z_{i+1}]$, for $i=1,\ldots, K-1$, and $F_{\Pi_\mathcal{C} \nu}(z_K) = 1$.
\end{restatable}

\begin{proof}
The consistency of the definition $\Pi_\mathcal{C} \nu = \sum_{i=1}^K \mathbb{E}_{w \sim\nu}[h_{z_i}(w)] \delta_{z_i}$ with \eqref{eq:cramerprojdiracs} follows immediately by observing directly that the definitions agree when $\nu$ is a Dirac measure, and then observing that the definition of $\Pi_\mathcal{C}$ in the statement of the proposition is also affine.

For the characterisation of $F_{\Pi_\mathcal{C} \nu}(z_i)$ for $i=1,\ldots, K-1$, we note that
\begin{align*}
F_{\Pi_\mathcal{C} \nu}(z_i) & = \sum_{j=1}^i \mathbb{E}_{w \sim\nu}[h_{z_j}(w)] \\
 & = \mathbb{E}_{w \sim\nu}\left[\sum_{j=1}^i h_{z_j}(w)\right] \\
  & = \mathbb{E}_{w \sim \nu}\left[ \mathbbm{1}_{w \leq z_i} + \mathbbm{1}_{w \in (z_i, z_{i+1}]} \frac{z_{i+1} - w}{z_{i+1} - z_i} \right] \\
  & = \frac{1}{z_{i+1} - z_i} \int_{z_i}^{z_{i+1}} F_\nu(w) \calcd w \, ,
\end{align*}
as required. Finally, since $\Pi_\mathcal{C} \nu$ is supported on $\{z_1, \ldots, z_K\}$, it immediately follows that $F_{\Pi_\mathcal{C} \nu}(z_K) = 1$.
\end{proof}

\section{Mixture update version of categorical policy evaluation and categorical Q-learning}\label{sec:mixturealgo}
Here we give a precise specification of the mixture update versions of categorical policy evaluation and categorical Q-learning, as described in the main paper in Section \ref{sec:SA}. The difference from Algorithm \ref{alg:c51} is highlighted in red.

\begin{algorithm}[H]
	\caption{CDRL mixture update}\label{alg:mixture}
	\begin{algorithmic}[1]
		\REQUIRE $\eta_t^{(x, a)} = \sum_{k=1}^K p^{(x, a)}_{t, k} \delta_{z_k}$ for each $(x, a)$
		\STATE Sample transition $(x_t, a_t, r_t, x_{t+1})$
		\STATE \textcolor{gray}{\# Compute distributional Bellman target}
		\IF{Categorical policy evaluation}
		\STATE $a^* \sim \pi(\cdot | x_{t+1})$
		\ELSIF{Categorical Q-learning}
		\STATE $a^* \leftarrow \argmax_a \mathbb{E}_{R \sim \eta_t^{(x_{t+1}, a)}}[R]$
		\ENDIF
		\STATE $\widehat{\eta}^{(x_t, a_t)}_* \leftarrow (f_{r_t, \gamma})_\# \eta_t^{(x_{t+1}, a^*)}$ 
		\STATE \textcolor{gray}{\# Project target onto support}
		\STATE $\widehat{\eta}^{(x_t, a_t)}_t \leftarrow \Pi_\mathcal{C} \widehat{\eta}_*^{(x_t, a_t)}$ 
		\STATE \textcolor{red}{\# Compute mixture update}
		\STATE \textcolor{red}{Generate new estimates according to mixture rule: $\eta_{t+1}^{(x_t, a_t)} = (1-\alpha_t(x_t, a_t)) \eta_t^{(x_t, a_t)} + \alpha_t(x_t,a_t) \widehat{\eta}^{(x_t, a_t)}_t$}
		\RETURN $\eta_{t+1}$
	\end{algorithmic}
\end{algorithm}

\section{Proof of results in Section \ref{sec:convergence}}

\CramProjNotWNonExpansion*

\begin{proof}
We exhibit a simple counterexample; it is enough to demonstrate that $\Pi_\mathcal{C}$ can act as an expansion. Take $z_1 = 0, z_2 = 1$, and consider two Dirac delta distributions, $\nu_1 = \delta_{1/4}$ and $\nu_2  = \delta_{3/4}$. We have $d_p(\nu_1, \nu_2) = ((1/2)^p)^{1/p} = 1/2$. Now $\Pi_\mathcal{C} \nu_1 = \frac{3}{4} \delta_0 + \frac{1}{4} \delta_1$, and $\Pi_\mathcal{C} \nu_2 = \frac{1}{4} \delta_0 + \frac{3}{4} \delta_1$, and hence $d_p(\Pi_\mathcal{C} \nu_1 , \Pi_\mathcal{C} \nu_2) = ((1/2) \times 1^p)^{1/p} = 2^{-1/p} > 1/2$.
\end{proof}

\projnonexpansion*

\begin{proof}
We begin by setting out a Hilbert space structure of a subset of $\mathscr{P}(\mathbb{R})$. Let $\mathcal{M}(\mathbb{R})$ be the vector space of all finite signed measures on $\mathbb{R}$. First, observe that the following subspace of signed measures:
\[
\mathcal{M}_0(\mathbb{R}) = \left\{ \nu \in \mathcal{M}(\mathbb{R}) \middle| \nu(\mathbb{R}) = 0 \, , \int_\mathbb{R} F_\nu(x)^2 \calcd x < \infty \right\} \, ,
\]
where $F_\nu(x) = \nu((-\infty, x])$ for each $x \in \mathbb{R}$, is isometrically isomorphic to a subspace of the Hilbert space $L^2(\mathbb{R})$ with inner product given by
\begin{align}\label{eq:innerprod}
\langle \nu_1, \nu_2 \rangle_{\ell_2} = \int_\mathbb{R} F_{\nu_1}(x) F_{\nu_2}(x) \calcd x \, .
\end{align}
Now consider the affine space $\delta_0 + \mathcal{M}_0(\mathbb{R})$ (i.e. the translation of $\mathcal{M}_0(\mathbb{R})$ in $\mathcal{M}(\mathbb{R})$ by the measure $\delta_0$). This affine space consists of signed measures of total mass $1$, with sufficiently quickly decaying tails. In particular, it contains the set of probability measures $\nu \in \mathscr{P}(\mathbb{R})$ satisfying
\[
\int_{-\infty}^0 F_\nu(x)^2 \calcd x < \infty \text{\ \ and \ \ } \int_{0}^\infty (1-F_\nu(x))^2 \calcd x < \infty \, .
\]
As $\delta_0 + \mathcal{M}_0(\mathbb{R})$ is an affine translation of a Hilbert space, it inherits the inner product defined in \eqref{eq:innerprod} from $\mathcal{M}_0(\mathbb{R})$, which is now defined for differences of elements. Now consider the affine subspace consisting of measures supported on $\{z_1, \ldots, z_K\}$. It is clear that this is a closed affine subspace (since it is finite-dimensional), and therefore there exists an orthogonal projection (with respect to the inner product defined above) onto this subspace, which we denote by $\Pi$. Given a probability measure $\nu \in \delta_0 + \mathcal{M}_0(\mathbb{R})$, $\Pi \nu = \sum_{i=1}^K p_i \delta_{z_i}$, where the $p_i$ satisfy $\sum_{i=1}^K p_i = 1$, and subject to this constraint, minimise $\langle \Pi \nu - \nu , \Pi\nu  - \nu \rangle_{\ell_2}$. But note that
\begin{align}\label{eq:l2min}
\langle \Pi \nu - \nu , \Pi\nu  - \nu \rangle_{\ell_2} = \int_\mathbb{R} (F_{\Pi\nu}(x) - F_\nu(x))^2 \calcd x \, .
\end{align}
By construction, $F_{\Pi\nu}$ is constant on the open intervals $(z_i, z_{i+1})$ for $i=1,\ldots, K-1$, and also on the intervals $(-\infty, z_1)$ and $(z_K, +\infty)$. Therefore $F_{\Pi\nu}$, and hence $\Pi\nu$ itself, is determined by the values of $F_{\Pi\nu}(z_i)$ for $i=1,\ldots, K$. The optimal values (i.e. those minimising \eqref{eq:l2min}) are easily verified to be: $F_{\Pi \nu}(z_K) = 1$, and $F_{\Pi\nu}(z_i)$ is equal to the average of $F_\nu$ on the interval $(z_i, z_{i+1})$, for $i=1,\ldots, K-1$. Note then that $\Pi \nu$ is a probability distribution (since $F_{\Pi\nu}$ is non-decreasing), and in fact matches the characterisation of $\Pi_\mathcal{C} \nu$ obtained in Proposition \ref{prop:characteriseproj}. Therefore we have established that $\Pi_\mathcal{C}$ is exactly orthogonal projection in the affine Hilbert space $\delta_0 + \mathcal{M}_0(\mathbb{R})$. Further, we have verified that the norm between elements in the space is exactly $\ell_2$, and hence it follows that $\Pi_\mathcal{C}$ is a non-expansion with respect to $\ell_2$.
\end{proof}

\pythag*

\begin{proof}
Denote by $F_\mu$, $F_{\Pi_\mathcal{C} \mu}$ and $F_{\nu}$ the CDFs of the measures $\mu$, $\Pi_\mathcal{C} \mu$ and $\nu$ respectively. Now note
\begin{align*}
    \ell_2^2(\mu, \nu) & = \int_{z_1}^{z_K} (F_\mu(x) - F_{\nu}(x))^2 \calcd x \\
    & = \int_{z_1}^{z_K} (F_\mu(x) - F_{\Pi_\mathcal{C} \mu}(x) + F_{\Pi_\mathcal{C} \mu}(x) - F_{\nu}(x))^2 \calcd x \\
    & = \int_{z_1}^{z_K} (F_\mu(x) - F_{\Pi_\mathcal{C}\mu}(x))^2 \calcd x + \int_{z_1}^{z_K} (F_\nu(x) - F_{\Pi_\mathcal{C}\mu}(x))^2 \calcd x \\
    &\qquad\qquad- 2\int_{z_1}^{z_K} (F_\mu(x) - F_{\Pi_\mathcal{C}\mu}(x))(F_\nu(x) - F_{\Pi_\mathcal{C}\mu}(x)) \calcd x \, .
\end{align*}
Finally, observe that
\begin{align*}
&\int_{z_1}^{z_K} (F_\mu(x) - F_{\Pi_\mathcal{C}\mu}(x))(F_\nu(x) - F_{\Pi_\mathcal{C}\mu}(x)) \calcd x \\
 = & \sum_{k=1}^{K-1} (F_{\nu}(z_k) - F_{\Pi_\mathcal{C}\mu}(z_k)) \int_{z_k}^{z_{k+1}} (F_\mu(x) - F_{\Pi_\mathcal{C}\mu}(x)) \calcd x \\
 = & 0 \, ,
\end{align*}
since by Proposition \ref{prop:characteriseproj}, $F_{\Pi_\mathcal{C} \mu}$ is constant on $(z_k, z_{k+1})$, and is equal to the average of $F_\mu$ on the same interval.
\end{proof}

\catcontract*

\begin{proof}
First, we show that the true distributional Bellman operator $\mathcal{T}^{\pi}$ is a $\sqrt{\gamma}$-contraction in $\overline{\ell}_2$. Note that through notions of scale sensitivity, as discussed by \citet{CramerGAN}, the ideas here may be extended to other distances over probability measures. Let $\eta, \mu \in \returndists$.
Then
\begin{align*}
\ell_2^2((\mathcal{T}^\pi \eta)^{(x,a)}, (\mathcal{T}^\pi \mu)^{(x,a)}) =& \ell_2^2\Bigg( \int_\mathbb{R} \sum_{(x^\prime, a^\prime) \in \mathcal{X} \times \mathcal{A}} \pi(a^\prime|x^\prime) p(\calcd r, x^\prime|x, a) (f_{r, \gamma})_\#\eta^{(x^\prime,a^\prime)}, \\
& \qquad \int_\mathbb{R} \sum_{(x^\prime, a^\prime) \in \mathcal{X} \times \mathcal{A}} \pi(a^\prime|x^\prime) p(\calcd r, x^\prime|x, a) (f_{r, \gamma})_\#\mu^{(x^\prime,a^\prime)}\Bigg) \\
& \leq \int_\mathbb{R}\sum_{(x^\prime, a^\prime) \in \mathcal{X} \times \mathcal{A}} \pi(a^\prime|x^\prime) p(\calcd r, x^\prime|x, a) \ell_2^2((f_{r, \gamma})_\#\eta^{(x^\prime, a^\prime)}, (f_{r, \gamma})_\#\mu^{(x^\prime, a^\prime)}) \\
& = \int_\mathbb{R}\sum_{(x^\prime, a^\prime) \in \mathcal{X} \times \mathcal{A}} \pi(a^\prime|x^\prime) p(\calcd r, x^\prime|x, a) \gamma \ell_2^2(\eta^{(x^\prime, a^\prime)}, \mu^{(x^\prime, a^\prime)}) \\
& \leq \gamma \overline{\ell}_2^2(\eta, \mu) \, ,
\end{align*}
with the first inequality following from Jensen's inequality, and the equality coming from the follow general fact about the Cram\'er distance and probability measures $\nu_1, \nu_2 \in \mathscr{}P(\mathbb{R})$:
\begin{align*}
\ell^2_2((f_{r,\gamma})_\# \nu_1, (f_{r,\gamma})_\# \nu_2) & = \int_{\mathbb{R}} (F_{(f_{r,\gamma})_\#\nu_1}(t) - F_{(f_{r,\gamma})_\#\nu_2}(t))^2 dt \\
& = \int_{\mathbb{R}} (F_{\nu_1}(f_{r,\gamma}^{-1}(t)) - F_{\nu_2}(f_{r,\gamma}^{-1}(t)))^2 dt \\
& = \int_{\mathbb{R}} \left(F_{\nu_1}\left(\frac{t-r}{\gamma}\right) - F_{\nu_2}\left(\frac{t-r}{\gamma}\right)\right)^2 dt \\
& = \gamma \int_{\mathbb{R}} \left(F_{\nu_1}\left(t^\prime\right) - F_{\nu_2}\left(t^\prime\right)\right)^2 dt^\prime \\
& = \gamma \ell_2^2(\nu_1, \nu_2)\, .
\end{align*}
Now by Proposition \ref{prop:projnonexpansion}, $\Pi_\mathcal{C}$ is a non-expansion in $\overline{\ell}_2$. Therefore $\Pi_\mathcal{C} \mathcal{T}^\pi$ is the composition of a non-expansion in $\overline{\ell}_2$ with a $\sqrt{\gamma}$-contraction in $\overline{\ell}_2$, and is therefore itself a $\sqrt{\gamma}$-contraction in $\overline{\ell}_2$. The second claim of the proposition then follows immediately from the Banach fixed point theorem.
\end{proof}

\cramerdistetac*

\begin{proof}
By Lemma \ref{lem:pythag}, we have:
\begin{align}\label{eq:recursion}
\overline{\ell}_2^2(\eta_\mathcal{C}, \eta_\pi) = & \sup_{(x,a) \in \mathcal{X} \times \mathcal{A}}\ell_2^2(\eta_\mathcal{C}^{(x, a)}, \eta_\pi^{(x, a)}) \nonumber \\
= & \sup_{(x,a) \in \mathcal{X} \times \mathcal{A}}\left\lbrack \ell_2^2(\eta_\mathcal{C}^{(x, a)}, (\Pi_\mathcal{C} \eta_\pi)^{(x, a)}) + \ell_2^2((\Pi_\mathcal{C}\eta_\pi)^{(x, a)}, \eta_\pi^{(x,a)}) \right\rbrack \nonumber \\
\leq &\ \overline{\ell}_2^2(\eta_\mathcal{C}, \Pi_\mathcal{C} \eta_\pi) + \overline{\ell}_2^2(\Pi_\mathcal{C} \eta_\pi, \eta_\pi) \nonumber\\
= &\  \overline{\ell}_2^2(\Pi_\mathcal{C} \mathcal{T}^\pi \eta_\mathcal{C}, \Pi_\mathcal{C} \mathcal{T}^\pi \eta_\pi) + \overline{\ell}_2^2(\Pi_\mathcal{C} \eta_\pi, \eta_\pi) \nonumber \\
\leq &\ \gamma \overline{\ell}_2^2(\eta_\mathcal{C}, \eta_\pi) + \overline{\ell}_2^2(\Pi_\mathcal{C} \eta_\pi, \eta_\pi)
\, ,
\end{align}
where in the final line we have used the contractivity of $\Pi_\mathcal{C} \mathcal{T}^\pi$ under $\overline{\ell}_2$ from Proposition \ref{prop:cat-contract}.
Due to Proposition \ref{prop:characteriseproj} (see Section \ref{sec:support}) we have that $F_{\Pi_\mathcal{C}\eta^{(x, a)}_\pi}$ is constant on the intervals $(z_i, z_{i+1})$ for $i=1,\ldots,K-1$, and moreover, due to the formula for the mass placed at the locations $z_{1:K}$, we also have
\[
F_{\Pi_\mathcal{C}\eta^{(x, a)}_\pi}(z_i) \in [ F_{\eta^{(x, a)}_\pi}(z_i), F_{\eta^{(x, a)}_\pi}(z_{i+1}) ] \quad \text{for }  i=1,\ldots, K-1 \quad \, , F_{\Pi_\mathcal{C} \eta^{(x, a)}_\pi}(z_K) = 1 \, .
\]
Therefore,
\begin{align*}
\ell_2^2(\Pi_\mathcal{C} \eta^{(x, a)}_\pi, \eta^{(x, a)}_\pi) & \leq \sum_{i=1}^{K-1} (z_{i+1} - z_i) (F_{\eta^{(x, a)}_\pi}(z_{i+1}) - F_{\eta^{(x, a)}_\pi}(z_i))^2 \\
& \leq  \left\lbrack \sup_{1\leq i < K} (z_{i+1} - z_i) \right\rbrack \sum_{i=1}^{K-1} (F_{\eta^{(x, a)}_\pi}(z_{i+1}) - F_{\eta^{(x, a)}_\pi}(z_i))^2  \\
& \leq  \left\lbrack \sup_{1\leq i < K} (z_{i+1} - z_i) \right\rbrack \left\lbrack \sum_{i=1}^{K-1} (F_{\eta^{(x, a)}_\pi}(z_{i+1}) - F_{\eta^{(x, a)}_\pi}(z_i))\right\rbrack^2 \\
& \leq \sup_{1\leq i < K} (z_{i+1} - z_i) \, ,
\end{align*}
for each $(x, a) \in \mathcal{X} \times \mathcal{A}$, yielding
\[
\overline{\ell}_2^2(\Pi_\mathcal{C} \eta_\pi, \eta_\pi) \leq \sup_{1\leq i < K} (z_{i+1} - z_i) \, .
\]

Thus, taking \eqref{eq:recursion}, applying the upper bound on $\overline{\ell}_2^2(\Pi_\mathcal{C} \eta_\pi, \eta_\pi)$ and rearranging, we obtain
\[
\overline{\ell}_2^2(\eta_\mathcal{C}, \eta_\pi)  \leq \frac{1}{1 - \gamma} \sup_{1\leq i < K} (z_{i+1} - z_i) \, .
\]
\end{proof}

\cramerdistmisspecified*

\begin{proof}
The proof proceeds as for that of Proposition \ref{prop:cramerdist-etaC}, obtaining the inequality
\[
\overline{\ell}_2^2(\eta_\mathcal{C}, \eta_\pi) \leq \frac{1}{1 - \gamma} \overline{\ell}_2^2(\Pi_\mathcal{C} \eta_\pi, \eta_\pi) \, .
\]
We now bound the right-hand side as follows:
\begin{align*}
\ell_2^2(\Pi_\mathcal{C} \eta^{(x, a)}_\pi, \eta^{(x, a)}_\pi) & \leq q^2 \times (z_1 - (z_1-\delta)) + q^2 ((z_K + \delta) - z_K) + \sum_{i=1}^{K-1} (z_{i+1} - z_i) (F_{\eta^{(x, a)}_\pi}(z_{i+1}) - F_{\eta^{(x, a)}_\pi}(z_i))^2 \\
& \leq 2q^2 \delta + \sup_{1\leq i < K} (z_{i+1} - z_i) \, ,
\end{align*}
which yields the result as required.
\end{proof}

\monotonebellman*

\begin{proof}
Let $\eta, \mu \in \returndists$, and suppose that $\eta \leq \mu$. This is equivalent to $F_{\eta^{(x, a)}} \geq F_{\mu^{(x, a)}}$ pointwise, for each $(x, a) \in \mathcal{X} \times \mathcal{A}$. We now compute the CDFs of $(\mathcal{T}^\pi \eta)^{(x, a)}$ and $(\mathcal{T}^\pi \mu)^{(x, a)}$, for each $(x, a) \in \returndists$, and show that stochastic dominance still holds. Indeed, by conditioning on the value of the tuple $(r, x^\prime, a^\prime)$, we obtain, for each
\begin{align*}
    (\mathcal{T}^\pi \eta)^{(x, a)}((-\infty, y]) & = \sum_{(x^\prime, a^\prime) \in \mathcal{X}\times\mathcal{A}} \int_{\mathbb{R}} p(\calcd r, x^\prime | x, a) \pi(a^\prime | x^\prime) (f_{r, \gamma})_{\#}\eta^{(x^\prime, a^\prime)}((-\infty, y]) \\
    & = \sum_{(x^\prime, a^\prime) \in \mathcal{X}\times\mathcal{A}} \int_{\mathbb{R}} p(\calcd r, x^\prime | x, a) \pi(a^\prime | x^\prime) \eta^{(x^\prime, a^\prime)}((-\infty, (y-r)/\gamma]) \\
    & \geq \sum_{(x^\prime, a^\prime) \in \mathcal{X}\times\mathcal{A}} \int_{\mathbb{R}} p(\calcd r, x^\prime | x, a) \pi(a^\prime | x^\prime) \mu^{(x^\prime, a^\prime)}((-\infty, (y-r)/\gamma]) \\
    & = \sum_{(x^\prime, a^\prime) \in \mathcal{X}\times\mathcal{A}} \int_{\mathbb{R}} p(\calcd r, x^\prime | x, a) \pi(a^\prime | x^\prime) (f_{r, \gamma})_{\#}\mu^{(x^\prime, a^\prime)}((-\infty, y]) \\
    & = (\mathcal{T}^\pi \mu)^{(x, a)}((-\infty, y]) \, ,
\end{align*}
as required, with the inequality coming from the fact that $\mu^{(x^\prime, a^\prime)}$ stochastically dominates $\eta^{(x^\prime, a^\prime)}$. This concludes the proof that the distributional Bellman operator $\mathcal{T}^\pi$ is monotone with respect to the partial order of element-wise stochastic dominance.

The monotonocity of the Cram\'er projection $\Pi_\mathcal{C}$ may be established from the expression given for the projection in Proposition \ref{prop:characteriseproj}. Suppose we have two distributions $\nu_1, \nu_2 \in \mathscr{P}(\mathbb{R})$, and suppose further that $\nu_1 \leq \nu_2$. Then recall from Proposition \ref{prop:characteriseproj} that we have $F_{\Pi_\mathcal{C} \nu_1}(w)$ and $F_{\Pi_\mathcal{C} \nu_2}(w)$ equal to $0$ for $w < z_1$ and equal to $1$ for $w \geq z_K$. For $w \in [z_i, z_{i+1})$ for some $i \in \{1,\ldots, K-1\}$, recall again from Proposition \ref{prop:characteriseproj} that we have
\begin{align}\label{eq:monotoneeqn}
F_{\Pi_\mathcal{C} \nu_j}(w) = \frac{1}{z_{i+1} - z_i}\int_{z_i}^{z_{i+1}} F_{\nu_j}(t) \calcd t \, , \quad \text{for\ } j=1,2 \, .
\end{align}
Since by assumption we have $F_{\nu_1} \geq F_{\nu_2}$ pointwise, it follows from \eqref{eq:monotoneeqn} that $F_{\Pi_\mathcal{C} \nu_1} \geq F_{\Pi_\mathcal{C} \nu_2}$ pointwise, and therefore $\Pi_\mathcal{C} \nu_1 \leq \Pi_\mathcal{C} \nu_2$, as required. 
\end{proof}

\subsection{Proof of Theorem \ref{thm:SAevaluation}}
\SAevaluation*

The proof structure is based on that of Theorem 2 of \citet{AsynchSAandQLearning}; our Lemmas \ref{lem:ULconvergence} and \ref{lem:sandwich2} are variants of Lemmas 5 and 6 of \citet{AsynchSAandQLearning}.
The high-level argument of the proof proceeds as follows.

Define:
\begin{align*}
    U^{(x, a)}_0 = \delta_{z_K} \, ,& \qquad L^{(x, a)}_0 = \delta_{z_1} \\
    U^{(x, a)}_{k+1} = \frac{1}{2} U^{(x, a)}_k + \frac{1}{2} (\Pi_\mathcal{C} \mathcal{T}^\pi U_k)^{(x, a)} \, ,& \qquad
    L^{(x, a)}_{k+1} = \frac{1}{2} L^{(x, a)}_k + \frac{1}{2} (\Pi_\mathcal{C} \mathcal{T}^\pi L_k)^{(x, a)} \, ,
\end{align*}
iteratively for each $(x, a) \in \mathcal{X} \times \mathcal{A}$.

\begin{restatable}{lemma}{ULconvergence}\label{lem:ULconvergence}
We have $U_{k+1} \leq U_k$, for each $k \in \mathbb{N}_0$, and $L_{k+1} \geq L_k$, for each $k \in \mathbb{N}_0$. Further, we have $U_k \rightarrow \eta_\mathcal{C}$ in  $\overline{\ell}_2$ , and also $L_k \rightarrow \eta_{\mathcal{C}}$ in $\overline{\ell}_2$.
\end{restatable}

Finally, we argue that, for each $k \in \mathbb{N}_0$, the return  distribution functions $U_k$ and $L_k$ sandwich all but finitely many of the return distribution estimators $\eta_t$, in a sense made precise by the following lemma.

\begin{restatable}{lemma}{sandwichtwo}\label{lem:sandwich2}
Given $k \in \mathbb{N}_0$, there exists a random time $T_k$ taking values in $\mathbb{N}_0$ such that
\[
L_{k} \leq \eta_t \leq U_k \ \ \text{\ for\ all\ } t > T_k, \text{\ almost\ surely}.
\]
\end{restatable}

Now, from Lemma \ref{lem:sandwich2} the conclusion of Theorem \ref{thm:SAevaluation} is reached as follows. Let $\varepsilon > 0$, and pick $k \in \mathbb{N}_0$ sufficiently large so that $\overline{\ell}_2(L_k, \eta_\mathcal{C}), \overline{\ell}_2(U_k, \eta_\mathcal{C}) < \varepsilon$, which can be done by Lemma \ref{lem:ULconvergence}. Note then by the triangle inequality that $\overline{\ell}_2(U_k, L_k) < 2\varepsilon$, and further, we have:
\begin{align*}
\overline{\ell_2}(\eta_t, \eta_\mathcal{C}) \leq \overline{\ell}_2(\eta_t, L_k) + \overline{\ell}_2(L_k, U_k) + \overline{\ell}_2(U_k, \eta_\mathcal{C})\, .    
\end{align*}
Since, by Lemma \ref{lem:sandwich2}, we have that $L_k \leq \eta_t \leq U_k$ for all $t > T_k$ almost surely, it follows that $\overline{\ell}_2(\eta_t, L_k) \leq \overline{\ell}_2(L_k, U_k)$ for all $t > T_k$ almost surely, and so we obtain
\begin{align*}
    \overline{\ell_2}(\eta_t, \eta_\mathcal{C}) \leq 2\overline{\ell}_2(L_k, U_k) + \overline{\ell}_2(U_k, \eta_\mathcal{C}) < 5\varepsilon \text{\ for\ all\ } t > T_k \text{\ almost\ surely}\, ,
\end{align*}
which yields the statement of Theorem \ref{thm:SAevaluation}. It now remains to establish Lemmas \ref{lem:ULconvergence} and \ref{lem:sandwich2}.

\subsection{Proof of Lemma \ref{lem:ULconvergence}}
We firstly show that $U_{k+1} \leq U_k$ for each $k \in \mathbb{N}_0$. The proof that $L_{k+1} \geq L_k$ for each $k \in \mathbb{N}_0$ is entirely analogous.

First, observe that $U_1 \leq U_0$, since each distribution $U_1^{(x, a)}$ is supported on $[z_1, z_K]$, and $U_0^{(x, a)}$ was chosen to stochastically dominate all distributions supported on $[z_1, z_K]$. For the inductive step, suppose $U_{k+1} \leq U_k$ for some $k \in \mathbb{N}_0$. Then by monotonicity of $\Pi_\mathcal{C} \mathcal{T}^\pi$, we have $\Pi_\mathcal{C} \mathcal{T}^\pi U_{k+1} \leq \Pi_\mathcal{C} \mathcal{T}^\pi U_k$. Hence,
\[
U^{(x, a)}_{k+2} = \frac{1}{2}U^{(x, a)}_{k+1} + \frac{1}{2}(\Pi_\mathcal{C} \mathcal{T}^\pi U_{k+1})^{(x, a)} \leq \frac{1}{2} U^{(x, a)}_k +  \frac{1}{2}(\Pi_\mathcal{C} \mathcal{T}^\pi U_{k})^{(x, a)} = U^{(x, a)}_{k+1}   \, ,
\]
which completes the inductive proof. To establish convergence of $U_k$ to $\eta_\mathcal{C}$, we make use of the following general result.

\begin{lemma}\label{lem:stocdomcvge}
Let $(\nu_k)_{k=0}^\infty$ be a sequence of probability measures over $\{z_1, \ldots, z_K\}$, with the property that $\nu_{k+1} \leq \nu_k$ for each $k \in \mathbb{N}_0$. Then there exists a probability measure $\nu^*$ over $\{z_1, \ldots, z_K\}$ such that $\nu_k \rightarrow \nu^*$ in $\ell_2$.
\end{lemma}
\begin{proof}
We work with CDFs. Denote the CDF of $\nu_k$ by $F_k$, for $k \in \mathbb{N}_0$. Recall that the stochastic dominance condition $\nu_{k+1} \leq \nu_k$ implies that $F_{k+1} \geq F_k$ pointwise. Therefore for each $x \in \mathbb{R}$, we have that $(F_k(x))_{k \in \mathbb{N}_0}$ is an increasing sequence, trivially upper-bounded by $1$. Therefore the sequence converges, and so there exists a limit function $F : \mathbb{R} \rightarrow \mathbb{R}$, defined by $F^*(x) = \lim_{k \rightarrow \infty} F_k(x)$. It is straightforward to see that this limit function takes values in $[0,1]$, is non-decreasing, right-continuous and is constant away from the set $\{z_1, \ldots, z_K\}$. It is therefore the CDF of a probability distribution $\nu^*$ supported on $\{z_1, \ldots, z_K\}$. Since $\widetilde{F}^*$ is constant away from $\{z_1,\ldots,z_K\}$, $\nu^*$ is supported on $\{z_1,\ldots,z_K\}$. To show that $\nu_k \rightarrow \nu^*$ in $\ell_2$, we must establish that $\int_\mathbb{R} (F_k(x) - F^*(x))^2 \calcd x \rightarrow 0$. Since $ \nu^* \leq \nu_{k+1} \leq \nu_k$ for each $k \in \mathbb{N}_0$, it follows that $\int_\mathbb{R} (F_k(x) - F^*(x))^2 \calcd x$ is a non-increasing sequence, and so it suffices to show that it is not lower-bounded by a positive number to establish the sequence's convergence to $0$. To that end, let $\varepsilon > 0$. Pick $k \in \mathbb{N}_0$ such that $|F_k(z_i) - F^*(z_i)| < \varepsilon$, for each $i=1,\ldots, K-1$. Then observe that
\[
\int_\mathbb{R} (F_k(x) - F^*(x))^2 \calcd x \leq \sum_{i=1}^{K-1} (z_{i+1} - z_i) \varepsilon^2 \, ,
\]
which demonstrates that no positive lower bounded exists, as required.
\end{proof}

Applying Lemma \ref{lem:stocdomcvge} to each of the sequences $(U_k^{(x, a)})_{k =0 }^\infty$, for each state-action pair $(x, a) \in \mathcal{X} \times \mathcal{A}$, we obtain the convergence of $(U_k)_{k=0}^\infty$ to some set of return distributions $\eta^*$ in $\overline{\ell}_2$. Finally, due to the continuity of $\Pi_\mathcal{C} \mathcal{T}^\pi$ with respect to $\overline{\ell}_2$, this limiting set of return distributions $\eta^*$ must satisfy $\eta^* = \frac{1}{2} \eta^* + \frac{1}{2} \Pi_\mathcal{C} \mathcal{T}^\pi \eta^*$, implying that $\eta^* = \Pi_\mathcal{C} \mathcal{T}^\pi \eta^*$, so the limiting set of return distributions is indeed the fixed point $\eta_\mathcal{C}$ of $\Pi_\mathcal{C} \mathcal{T}^\pi$. Analogously, we may show that $L_k \rightarrow \eta_\mathcal{C}$ in $\overline{\ell}_2$.

\subsection{Proof of Lemma \ref{lem:sandwich2}}
We prove this lemma by induction.
The result is clear for $k=0$, as in this case $U_0^{(x, a)}$ stochastically dominates all distributions supported on $[z_1,z_K]$, and $L_0^{(x, a)}$ is stochastically dominated by all distributions supported on $[z_1, z_K]$.
Now assume the result holds for some $k\geq0$; that is, there exists some random time $T_k$ such that $L_k \leq \eta_t \leq U_k$ for all $t \geq T_k$ almost surely. Here, we follow the structure of the proof of Lemma 6 of \citep{AsynchSAandQLearning} closely. We will show there exists a random time $T_{k+1}$ such that $\eta_t \leq U_{k+1}$ for all $t \geq T_{k+1}$ almost surely; the claim that $L_{k+1} \leq \eta_t$ for all $t \geq T_{k+1}$ may be proven analogously.

Now define 
\begin{align}\label{eq:HandWdef}
    H_{T_k}^{(x, a)} = U^{(x, a)}_k &\, , \ H^{(x, a)}_{t+1} = (1 - \alpha_t(x, a)) H^{(x, a)}_t + \alpha_t(x, a) (\Pi_\mathcal{C} \mathcal{T}^\pi U_k)^{(x, a)} \, , \text{\ for\ } t \geq T_k  \\
    W_{T_k}^{(x, a)} = 0 \in \mathcal{M}(\mathbb{R}) &\, , \ W_{t+1}^{(x, a)} = (1 - \alpha_t(x, a)) W_t^{(x, a)} + \alpha_t(x, a) \left\lbrack (\Pi_\mathcal{C}(f_{r, \gamma})_\# \eta_t)^{(x^\prime, a^\prime)} - (\Pi_\mathcal{C} \mathcal{T}^\pi \eta_t)^{(x, a)} \right\rbrack , \text{for\ } t \geq T_k \, , \nonumber
\end{align}
where $\mathcal{M}(\mathbb{R})$ is the space of signed measures on $\mathbb{R}$, and $0 \in \mathcal{M}(\mathbb{R})$ represents the zero measure; that is, the signed measure that assigns measure $0$ to every Borel subset of $\mathbb{R}$. Note that the process $(W_t)_{t \geq T_k}$ takes values in the space of collections of finite signed measures indexed by state-action pairs, each with overall mass $0$; that is, $W^{(x, a)}_t(\mathbb{R}) = 0$ for all $(x, a) \in \mathcal{X} \times \mathcal{A}$, for all $t \geq T_k$.

We now argue that $\eta_t^{(x, a)} \leq H^{(x, a)}_t + W_t^{(x, a)}$ for all $t \geq T_k$ and for all $(x, a) \in \mathcal{X} \times \mathcal{A}$ almost surely. For $t=T_k$, this following from the definitions in \eqref{eq:HandWdef} and the dominance relation $\eta_{T_k} \leq U_k$. To complete the proof, we proceed inductively. Suppose that $\eta_t^{(x, a)} \leq H^{(x, a)}_t + W_t^{(x, a)}$ for all $(x, a) \in \mathcal{X} \times \mathcal{A}$, for some $t \geq T_k$. Then note, assuming $\alpha_t(x, a)=0$ if the distribution corresponding to the state-action pair $(x, a)$ is not updated at time $t$, we have
\begin{align*}
\eta_{t+1}^{(x, a)} = & (1 - \alpha_t(x ,a)) \eta_t^{(x, a)} + \alpha_t(x, a) \Pi_\mathcal{C} (f_{r, \gamma})_\# \eta_t^{(x^\prime, a^\prime)} \\
= &(1 - \alpha_t(x ,a)) \eta_t^{(x, a)} + \alpha_t(x, a) (\Pi_\mathcal{C} \mathcal{T}^\pi \eta_t)^{(x, a)} + \alpha_t(x,a) (\Pi_\mathcal{C}(f_{r, \gamma})_\# \eta_t^{(x^\prime, a^\prime)} - (\Pi_\mathcal{C} \mathcal{T}^\pi \eta_t)^{(x, a)}) \\
\overset{(i)}{\leq} &(1 - \alpha_t(x ,a)) (H_t^{(x, a)} + W_t^{(x, a)}) + \alpha_t(x, a) (\Pi_\mathcal{C} \mathcal{T}^\pi U_k)^{(x, a)} + \alpha_t(x,a) (\Pi_\mathcal{C}(f_{r, \gamma})_\# \eta_t^{(x^\prime, a^\prime)} - (\Pi_\mathcal{C} \mathcal{T}^\pi \eta_t)^{(x, a)}) \\
= &(1 - \alpha_t(x ,a)) H_t^{(x, a)} + \alpha_t(x, a) (\Pi_\mathcal{C} \mathcal{T}^\pi U_k)^{(x, a)} +(1- \alpha_t(x, a)) W_t^{(x, a)} \\
& \qquad\qquad\qquad\qquad\qquad\qquad\qquad\qquad\qquad\qquad\qquad+ \alpha_t(x,a) (\Pi_\mathcal{C}(f_{r, \gamma})_\# \eta_t^{(x^\prime, a^\prime)} - (\Pi_\mathcal{C} \mathcal{T}^\pi \eta_t)^{(x, a)}) \\
= & H_{t+1}^{(x, a)} + W_{t+1}^{(x, a)}\, , 
\end{align*}
as required. In the above derivation, (i) comes from the stochastic dominance relations $\eta_t \leq H_t + W_t$ (by induction hypothesis) and $\eta_t \leq U_k$ and the monotonicity of $\Pi_\mathcal{C} \mathcal{T}^\pi$. Note that we have the following expression for $H^{(x, a)}_t$:
\[
H_t^{(x, a)} = \left(\prod_{\tau = T_k}^{t-1} (1 - \alpha_\tau(x, a))\right) U_k  + \left( 1- \prod_{\tau = T_k}^{t-1} (1 - \alpha_\tau(x, a))\right) (\Pi_\mathcal{C} \mathcal{T}^\pi U_k)^{(x, a)}
\]
Since by assumption we have $\sum_{k=0}^\infty \alpha_k(x, a) = \infty$ for all $(x, a) \in \mathcal{X} \times \mathcal{A}$ almost surely, we have that there exists a random time $\widetilde{T}_{k+1}$ such that $\prod_{\tau = T_k}^{t-1} (1 - \alpha_\tau(x, a)) \leq 1/4$ for all $(x, a) \in \mathcal{X} \times \mathcal{A}$, and for all $t \geq \widetilde{T}_{k+1}$ almost surely. Since $\Pi_\mathcal{C} \mathcal{T}^\pi U_k \leq U_k$, for all $t \geq \widetilde{T}_k$, we have:
\begin{align}
\eta_{t} & \leq H_t + W_t \nonumber \\
& \leq \frac{1}{4} U_k + \frac{3}{4} \Pi_\mathcal{C} \mathcal{T}^\pi U_k + W_t \nonumber \\
& = \frac{1}{2}U_k + \frac{1}{2} \Pi_\mathcal{C} \mathcal{T}^\pi U_k + W_t - \frac{1}{4}( U_k - \Pi_\mathcal{C} \mathcal{T}^\pi U_k) \nonumber \\
& = U_{k+1} + W_t - \frac{1}{4}( U_k - \Pi_\mathcal{C} \mathcal{T}^\pi U_k) \label{eq:eta_U_W_ineq} \, .
\end{align}
Now note that if $U^{(x, a)}_k((\infty, z_i]) = \Pi_\mathcal{C} \mathcal{T}^\pi U_k^{(x, a)}((\infty, z_i])$, then we have $U_{k+1}^{(x, a)}((-\infty, z_i]) = U_k^{(x, a)}((-\infty, z_i])$. Let $\delta$, then, be the smallest non-zero value of $|(\Pi_\mathcal{C} \mathcal{T}^\pi U_{k})^{(x, a)}((-\infty, z_i]) - U_k^{(x, a)}((-\infty, z_i])|$ across all state-action pairs $(x, a) \in \mathcal{X} \times \mathcal{A}$ and all support points $z_i \in \{z_1, \ldots, z_K\}$.
Crucially, we observe that the additive ``noise" term appearing in the definition of $W_{t+1}^{(x, a)}$ in Equation \eqref{eq:HandWdef} is mean-zero, in the following sense: as a random measure, the expectation of the noise term is the $0$ measure. More concretely for our purposes, we have, as stated in Lemma \ref{lem:signedmeasure} in the main paper, for all $z_i \in \{z_1, \ldots, z_K\}$:
\begin{align*}
    \mathbb{E}_{r, x^\prime, a^\prime}\left\lbrack ((\Pi_\mathcal{C}(f_{r, \gamma})_\# \eta_t)^{(x^\prime, a^\prime)} - (\Pi_\mathcal{C} \mathcal{T}^\pi \eta_t)^{(x, a)}) \right\rbrack((-\infty, z_i]) = 0 \, .
\end{align*}
Standard stochastic approximation theory (e.g. \citep{AsynchSAandQLearning}), via Assumption (i), then yields that $W_t^{(x, a)}((-\infty, z_i]) \rightarrow 0$ almost surely, for all $(x, a) \in \mathcal{X} \times \mathcal{A}$, and for all $z_i \in \{z_1, \dots, z_K\}$. We can now take $T_{k+1} > \widetilde{T}_{k+1}$ sufficiently large so that $|W_t^{(x, a)}((-\infty, z_i]))| < \delta/4$ for all $t \geq T_{k+1}$ and all $(x, a) \in \mathcal{X} \times \mathcal{A}$. Then \eqref{eq:eta_U_W_ineq} yields that $\eta_t \leq U_{k+1}$ for all $t \geq T_{k+1}$, completing the inductive step, and therefore completing the proof of Lemma \ref{lem:sandwich2}.

\subsection{Proof of Theorem \ref{thm:SAcontrol}}

\SAcontrol*
\begin{proof}
We first note that the updates induced by the algorithm on the \textit{expected returns} are exactly those of standard (non-distributional) Q-learning. More precisely, denoting the expected returns $\mathbb{E}_{R \sim \eta_t^{(x, a)}}[R]$ at state-action pair $(x, a) \in \mathcal{X} \times \mathcal{A}$ at time $t$ by $Q_t(x, a)$, we have that these Q-values follow the standard dynamics of Q-learning. This holds because the maximum and minimum possible estimated rewards lie within the support of the parametrised distributions, by the assumptions of the theorem. We may therefore apply the non-distributional theory \citep{AsynchSAandQLearning} to argue that the expectations $(Q_t(x, a) | (x, a) \in \mathcal{X} \times \mathcal{A})$ converge almost-surely to the true optimal expected returns $(Q^{\pi^*}(x,a) | (x,a) \in \mathcal{X} \times\mathcal{A})$. Since the state space and action space are finite, this convergence is almost-surely uniform across all state-action pairs. Therefore, given $\varepsilon > 0$, there exists a random variable $N$ such that for $t > N$, we have
\[
\sup_{(x, a) \in \mathcal{X} \times \mathcal{A}} | Q_t(x, a) - Q^{\pi^*}(x, a) | < \varepsilon \qquad \text{almost surely} \, .
\]
Now take $\varepsilon$ to be equal to half the minimum action gap across all states for the optimal action-value function $Q^{\pi^*}$; that is, take $\varepsilon = \frac{1}{2} \min_{x \in \mathcal{X}} [Q^{\pi^*}(x, \pi^*(x)) - \max_{a \not= \pi^*(x)} Q^{\pi^*}(x, a)]$ (which is greater than zero by the assumption of a unique optimal policy and finite state and action spaces). Then for $t > N$, the Q-learning updates are exactly the same as policy evaluation updates for the optimal policy $\pi^*$. Under these updates, we proved in Theorem \ref{thm:SAevaluation} that the return distributions converge to the approximate return distribution function $\eta_\mathcal{C}$. Note however, that $N$ is not a stopping time; we must be particularly careful with the analysis that follows.

We therefore proceed according to a coupling argument. We define the following set of independent stochastic distributional Bellman operators: $(\widehat{\mathcal{T}}^{\pi}_t)$ across all deterministic policies $\pi$, and timesteps $t \in \mathbb{N}$. The idea is to define a $\pi^*$ categorical policy evaluation algorithm with these operators, and also a categorical Q-learning algorithm, and couple these processes together with probability tending to $1$ as the number of steps of each algorithm increases. Since the return distribution ensemble computed by the policy evaluation algorithm will converge to the approximate return distribution function $\eta_\mathcal{C}$ associated with $\pi^*$ almost surely, we will then be able to argue that the same is true of the distributions computed by the Q-learning algorithm.

More precisely, we first construct the $\pi^*$ categorical policy evaluation algorithm by taking an initial return distribution function $(\eta_0^{(x,a)} | (x,a) \in \mathcal{X} \times \mathcal{A})$, and defining:
\[
\eta_{k+1} =\Pi_\mathcal{C} \widehat{\mathcal{T}}_k^{\pi^*} \eta_k \, ,
\]
for each $k \geq 0$. We construct the Q-learning algorithm by taking the same initial return distribution function $(\eta_0^{(x, a)} | (x,a) \in \mathcal{X} \times \mathcal{A})$, and defining the following updates, letting $\widetilde{\eta}_0 = \eta_0$:
\begin{align*}
&\text{Let } \pi_k \text{ be greedy wrt } \widetilde{\eta}_k \, , \\
&\widetilde{\eta}_{k+1} = \Pi_\mathcal{C} \widehat{\mathcal{T}}_k^{\pi_k} \widetilde{\eta}_k \, , 
\end{align*}
for each $k \geq 0$.

By the remarks above, we have $\pi_k = \pi^*$ for all $k$ sufficiently large almost surely. Let $A_k = \{ \pi_l = \pi^* \text{ for all } l \geq k \}$, for each $k \in \mathbb{N}$. Then $A_{k} \subseteq A_{k+1}$, and $\mathbb{P}(A_k) \uparrow 1$. Let $B$ be the event of probability $1$ for which the policy evaluation algorithm converges. Now, on the event $B \cap A_k$, we have
\[
\overline{\ell}_2^2(\eta_l, \eta_\mathcal{C}) \rightarrow 0 \, ,
\]
where $\eta_\mathcal{C}$ is the limiting distribution function for the policy $\pi^*$, as in Theorem \ref{thm:SAevaluation}.
Note then that if $\overline{\ell}_2^2(\widetilde{\eta}_l, \eta_l) \rightarrow 0$ on this event too, then by the triangle inequality, we have $\overline{\ell}_2(\widetilde{\eta}_l, \eta_\mathcal{C}) \rightarrow 0$, and hence Q-learning converges on $A_k \cap B$, and since $\mathbb{P}(A_k \cap B) \uparrow 1$, the statement of the theorem immediately follows. We first observe that 
$\overline{\ell}_2^2(\widetilde{\eta}_l, \eta_l)$, for $l \geq k$, is eventually a non-increasing positive sequence on the event $A_k$:
\begin{align}
    &\ell_2^2(\widetilde{\eta}_{l+1}^{(x, a)}, \eta_{l+1}^{(x, a)}) \nonumber
    \\
    = & \left\|\left( (1-\alpha_l(x,a))\widetilde{\eta}_l^{(x, a)} + \alpha_l(x, a)( \Pi_\mathcal{C} \widehat{\mathcal{T}}^{\pi^*}_l \widetilde{\eta}_l)^{(x,a)} \right)- \left((1-\alpha_l(x,a)) \eta_l^{(x, a)} + \alpha_l(x, a) (\Pi_\mathcal{C} \widehat{\mathcal{T}}^{\pi^*}_l \eta_l)^{(x,a)}\right) \right\|_{\ell_2}^2 \nonumber \\
    = & (1 - \alpha_l(x, a))^2 \left\|\widetilde{\eta}_l^{(x, a)} - \eta_l^{(x, a)}\right\|_{\ell_2}^2 +
    \alpha_l(x, a)^2 \left\| (\Pi_\mathcal{C} \widehat{\mathcal{T}}^{\pi^*}_l \widetilde{\eta}_l)^{(x,a)} - (\Pi_\mathcal{C} \widehat{\mathcal{T}}^{\pi^*}_l \eta_l)^{(x,a)}  \right\|_{\ell_2}^2 \nonumber \\
     & \qquad\qquad\qquad\qquad\qquad\qquad + 2\alpha_l(x, a)(1 - \alpha_l(x, a)) \langle\widetilde{\eta}_l^{(x, a)} - \eta_l^{(x, a)}, (\Pi_\mathcal{C} \widehat{\mathcal{T}}^{\pi^*}_l \widetilde{\eta}_l)^{(x,a)} - (\Pi_\mathcal{C} \widehat{\mathcal{T}}^{\pi^*}_l \eta_l)^{(x,a)}  \rangle_{\ell_2} \nonumber \\
    \leq & (1- \alpha_l(x, a))^2 \overline{\ell}_2^2(\widetilde{\eta}_{l}, \eta_{l}) + \alpha_l(x, a)^2 \gamma \overline{\ell}_2^2(\widetilde{\eta}_{l}, \eta_{l}) + 
    2\alpha_l(x, a) (1 - \alpha_l(x, a)) \sqrt{\gamma} \overline{\ell}_2^2(\widetilde{\eta}_{l}, \eta_{l}) \nonumber \\
    = & (1 - \alpha_l(x, a)(1 - \sqrt{\gamma}))^2 \overline{\ell}_2^2(\widetilde{\eta}_{l}, \eta_{l}) \label{eq:bound} \, .
\end{align}
Therefore, on this event, $\overline{\ell}_2(\widetilde{\eta}_l, \eta_l)$ has a limit almost surely. Denote $Z$ as the limit of the sequence, and on the event that $Z > 0$, pick $\delta > 0$ such that $\sqrt{\gamma}(Z + \delta) < Z$.
Letting $\tau > 0$ such that $\overline{\ell}_2(\widetilde{\eta}_l, \eta_l) < Z + \delta$ for all $l \geq \tau$, we observe that for $l \geq \tau$, following the calculations in Equation \eqref{eq:bound}, we obtain the inequality
\begin{align*}
\ell_2^2(\widetilde{\eta}_{l+1}^{(x, a)}, \eta_{l+1}^{(x, a)}) & \leq (1 - \alpha_l(x, a))^2 \ell_2^2(\widetilde{\eta}_{l}^{(x, a)}, \eta_{l}^{(x, a)}) + \alpha_l(x, a)^2 \gamma (Z + \delta) + 2\alpha_l(x, a) (1 - \alpha_l(x, a)) \sqrt{\gamma} (Z + \delta) \\
& \leq (1 - 2\alpha_l(x, a) + \alpha_l(x, a)^2) \ell_2^2(\widetilde{\eta}_{l}^{(x, a)}, \eta_{l}^{(x, a)}) + (2\alpha_l(x, a) - \alpha_l(x, a)^2)\sqrt{\gamma}(Z + \delta) \, .
\end{align*}
By Assumption \ref{cond:RM} of the theorem, we have $\lim \sup_l \ell_2(\widetilde{\eta}_{l}^{(x, a)}, \eta^{(x, a)}_{l}) \leq \sqrt{\gamma}(Z+\delta) < Z$ for all $(x,a ) \in \mathcal{X} \times \mathcal{A}$, a contradiction. Therefore $\overline{\ell}_2^2(\widetilde{\eta}_l, \eta_l) \rightarrow 0$ holds on $A_k \cap B$ almost surely, as required.
\end{proof}

\end{document}